\documentclass{article} 
\usepackage{iclr2026_conference,times}

\usepackage{hyperref}
\usepackage{url}
\usepackage{amsmath}
\usepackage{amssymb}
\usepackage{mathtools}
\usepackage{amsthm}
\usepackage{thmtools}
\usepackage{thm-restate}
\usepackage{chngcntr}

\usepackage[capitalize,noabbrev]{cleveref}

\usepackage{booktabs}
\usepackage{multirow}
\usepackage[table,xcdraw]{xcolor}
\usepackage{algorithm,algorithmic}
\usepackage{placeins}
\usepackage{capt-of}

\newtheorem{remark}{\textbf{Remark}}

\usepackage[textsize=tiny]{todonotes}

\newcommand{\bg}{\boldsymbol{g}}

\newcommand{\bone}{\boldsymbol{1}}

\newcommand{\bbP}{\mathbb{P}}

\newcommand{\mE}{\mathbb{E}}

\newcommand{\cD}{\mathcal{D}}

\newcommand{\cV}{\mathcal{V}}

\newcommand{\cA}{\mathcal{A}}

\newcommand{\cJ}{\mathcal{J}}

\makeatother

\newcommand{\think}[1]{\textcolor{blue}{\texttt{<think>}} #1 \textcolor{blue}{\texttt{</think>}}}
\newcommand{\search}[1]{\textcolor{cyan}{\texttt{<search>}} #1 \textcolor{cyan}{\texttt{</search>}}}
\newcommand{\info}[1]{\textcolor{brown}{\texttt{<information>}} #1 \textcolor{brown}{\texttt{</information>}}}
\newcommand{\answer}[1]{\textcolor{purple}{\texttt{<answer>}} #1 \textcolor{purple}{\texttt{</answer>}}}

\title{Turning Off-Policy Tokens On-Policy: A Plug-in Approach for Improving LLM Alignment}

\iclrpreprint

\renewcommand\footnotemark{}
\author{
    \!\!Yu Li$^{1,*}$, Xiuyu Li$^{1,*,\ddagger}$,  Mingyang Yi$^{1,\dagger}$, \\
    \textbf{Jiaxing Wang$^2$, Liangxu Zhang$^2$, Zhaolong Xing$^2$, Zhen Chen$^2$}\\
    \thanks{$^*$Equal contribution.}
    \thanks{$^\dagger$Corresponding author.}
    \thanks{The code is available at \url{https://github.com/liyu199809/SIS/tree/main}}\\
    $^1$Renmin University of China \quad $^2$JD.com \\
    \footnotesize{\texttt{{\{liyu0929,yimingyang\}@ruc.edu.cn}}} \\
}

\begin{document}

\maketitle

\begin{abstract}
Reinforcement learning (RL) post-training for large language models (LLMs) follows a efficient paradigm of ``rollout then update'', which inevitably results in off-policy training data. 
To resolve this, Importance sampling (IS) is proposed, while the token-level ratios compound over long sequences, causing severe variance exploded.
A natural idea is ``transferring'' these off-policy token into on-policy token, so that the importance scores for correction are unnecessary. Following this idea,  
we propose \emph{Selective Importance Sampling} (SIS), which is inspired by rejection sampling. Concretely, SIS implements by viewing off-policy model as proposal distribution, and implement a token-level rejection test: accepted tokens are viewed as on-policy, so that receive unit importance score, while rejected tokens retain the standard IS correction. 
Our proposed SIS is theoretically proved  reducing the gap between token-level and sequence-level off-policy gradient estimators. 
The SIS acts as a plug-in that only modifies the importance ratio in the policy loss, adding negligible wall-clock overhead, and can be combine with a vast vary of RL post-training algorithms. 
Experiments on dense and MoE LLMs across math and agent benchmarks show that SIS consistently improves all objectives, 
while providing substantially stronger robustness under off-policy data.
\end{abstract}

\section{Introduction}

Reinforcement learning (RL) \citep{sutton1998reinforcement} has become the dominant paradigm for post-training large language models, driving substantial progress in reasoning \citep{guo2025deepseek, chen2025minimax}, code generation \citep{jiang2025coderl+}, tool use \citep{li2026reasoning}, and broader capabilities \citep{li2026ets, li2026towards,zhang2026htam,zhang2026reward,tu2026ucob}. 
The RL process follows an order of ``rollout then update''.
For efficiency, techniques such as sample reuse \citep{noukhovitch2025faster}, asynchronous rollouts \citep{fu2026areal} are widely adopted, but inevitably introduce off-policy data from stale policies.
The resulting distribution mismatch introduces bias into gradient estimates and degrades training stability \citep{ma2025stabilizing}. This makes effective off-policy gradient correction a fundamental challenge for RL-based post-training methods.

Importance sampling (IS) \citep{tokdar2010importance} offers a principled correction to off-policy gradient by multiplying a correction ratio, but its variance can grow rapidly in long-horizon reasoning \citep{metelli2020importance}, and is sensitive to some outliers \citep{schulman2017ppo}. 
To mitigate this, mainstream methods \citep{schulman2017ppo,shao2024deepseekmath,yu2025dapo} rely on hard clipping. 
Although effective for stabilization, hard clipping suppresses gradients from highly off-policy samples and loses useful learning signal. 
Some recent methods \citep{zheng2025group, chen2025minimax, gao2025soft} mitigate this issue by soft clipping, but remain heuristic.

Instead of working on clipping method, we address this problem by ``transferring'' the off-policy tokens into on-policy tokens, which fundamentally resolve the off-policy problem. Our method, \emph{Selective Importance Sampling} (SIS), a simple and effective method conduct this transfer via acceptance-rejection sampling \citep{robert2004monte}.  
Concretely, SIS views the behavior policy create training data as a proposal distribution to the current policy. By implementing a rejection sampling, i.e., a token-level acceptance test,  
the accepted off-policy tokens from behavior policy \textbf{can be viewed as on-policy tokens} and receive unit weight; rejected tokens retain the standard IS correction.
To further improve our SIS, we propose an efficient approximation to the acceptance test process, making our SIS brought negligible efforts in both computation and efficiency.   
More than the practical method, we further prove that our SIS strictly reduces the cumulative distributional mismatch across the sequence \citep{ma2025stabilizing}, tightening the off-policy approximation error bound, so that theoretically improves the current post-training methods. 

Empirically, SIS delivers consistent gains as a plug-in module across diverse backbones models (dense and Mixture-of-Experts \citep{yang2025qwen3}), RL algorithms, and benchmarks (math \citep{maa2024aime} and agent tasks \citep{jin2025search}).
Moreover, it enhances training stability under challenging off-policy regimes, and ablation studies confirm that it is insensitive to hyperparameter choices.

Our contributions are as follows:
\begin{itemize}
\item We propose SIS, a plug-in method that can be combined with most mainstream RL post-training methods e.g., GRPO \citep{shao2024deepseekmath}, DAPO \citep{yu2025dapo}, GSPO \citep{zheng2025group}. The method mitigates off-policy mismatch in LLM post-training in a rejection sampling regime. 
\item We provide theoretical justification for SIS, proving that its token-level rejection mechanism reduces off-policy approximation error.
\item We conduct extensive experiments spanning diverse model architectures, RL algorithms, and task domains, demonstrating that SIS delivers consistent gains 
and stronger stability under challenging off-policy regimes.
\end{itemize}

\section{Background}
\label{sec:pre}
In this section, we present the necessary background knowledge of this paper.

\textbf{Off-Policy Policy Gradient.}
Given a query $x$ from the dataset $\cD$, the response denoted as $y = (y_1, y_2, \dots, y_T)$ is sampled from LLM $\pi_{\theta}(y\mid x)$. The response is evaluated by a reward model $r(x, y)$. To optimize the policy LLM, we usually maximize the following regularized expected reward.\footnote{Here we ignore the KL regularization \citep{ouyang2022training} to simplify the notation, our analysis can be similarly conducted with it involved.} 
\begin{equation}
\small
\mathcal J(\theta) = \mathbb{E}_{x\sim \cD, y \sim \pi_\theta(\cdot \mid x)}[r(x, y)],
\end{equation}
The policy gradient of this objective with respect to the policy parameters can be computed via the log-trick \citep{sutton1998reinforcement} by 
\begin{equation}
\small
\nabla_\theta \cJ(\theta) = \mathbb{E}_{x\sim \cD, y \sim \pi_\theta(\cdot \mid x)} \left[\left(A(x, y)\right) \nabla_\theta \log \pi_\theta(y \mid x)  \right],
\end{equation}
\normalsize
where $A(x, y)$ denotes the advantage associated with response $y$ under prompt $x$, which results in unbiased gradient estimator e.g., $A(x, y) = r(x, y) - c$ for some constant $c$ \citep{schulman2017ppo}.

In typical large-scale settings, due to the limitations of sample reuse, rollout efficiency e.t.c. \citep{fu2026areal}, online sampling from $\pi_\theta$ to estimate policy gradient is infeasible, and training is performed on a fixed set of responses collected from a stale behavior policy $\pi_{\theta_{\mathrm{old}}}$. To reproduce the unbiased gradient estimator, we usually use importance sampling correction ratio $w(\theta) = \frac{\pi_\theta(y \mid x)}{\pi_{\theta_{\mathrm{old}}}(y \mid x)}$
\begin{equation}\label{eq:ground truth gradient}
\small
\begin{aligned}
\nabla_\theta \cJ(\theta) = \mathbb{E}_{x\sim \cD, y \sim \pi_{\theta_{\mathrm{old}}(\cdot \mid x)}} \left[w(\theta)A(x, y) \nabla_\theta \log \pi_\theta(y \mid x)  \right].
\end{aligned}
\end{equation}
The sequence-level correction ratio $w(\theta)$ can be viewed as a factor to account for the discrepancy between the target and behavior policies. 
In practice, this ratio may become excessively large, leading to exploding gradients and unstable training \citep{schulman2017ppo,yu2025dapo}. 
To mitigate this, $w(\theta)$ is typically modified through clipping or normalization to ensure more stable training as proposed in \citep{schulman2017ppo,zheng2025group}.

\textbf{Token-level Importance Sampling.}
For autoregressive models, the sequence likelihood admits the following factorization
\begin{equation}
\small
    \pi_\theta(y \mid x) = \prod_{t=1}^{T} \pi_\theta(y_t \mid x, y_{<t}).
\end{equation}
Consequently, the sequence-level ratio $w(\theta)$ decomposes into a product of token-level ratios
\begin{equation}\label{eq:ratio}
\small
    w(\theta) = \prod_{t=1}^{T} w_t(\theta), \quad
    w_t(\theta) = \frac{\pi_\theta(y_t \mid x, y_{<t})}{\pi_{\theta_{\mathrm{old}}}(y_t \mid x, y_{<t})}.
\end{equation}
This product scales exponentially with sequence length, leading to severe variance in gradient estimation \citep{metelli2020importance}. 
To mitigate this, practical implementations replace the sequence-level importance ratio in gradient estimation by single token-level ratios \citep{ma2025stabilizing} and results in approximated gradient 
\begin{equation}\label{eq:gradient}
\small
\tilde{g}(\theta) = \mathbb{E}_{y \sim \pi_{\theta_{\mathrm{old}}}(\cdot \mid x)} \left[  \sum_{t=1}^{T} w_t(\theta) A(x, y)\nabla_\theta \log \pi_\theta(y_t \mid x, y_{<t}) \right].
\end{equation}
Although $\tilde{g}(\theta)$ does not equal the true policy gradient $\nabla_\theta \cJ(\theta)$, it provides a reliable approximation as long as $\pi_\theta$ stays sufficiently close to $\pi_{\theta_{\mathrm{old}}}$\citep{kakade2002approximately,zheng2025stabilizing}, since $w_{t}(\theta) = \pi_\theta(y_t \mid x, y_{<t})/\pi_{\theta_{\mathrm{old}}}(y_t \mid x, y_{<t})\approx 1$. 
When the two policies diverge, this approximation gap grows, producing biased gradient estimates that destabilize training.
Controlling this gap is therefore central to reliable off-policy LLM training; we formalize this in Section~\ref{sec:theoretical_analysis}.

\section{Method: Turning Off-Policy Tokens On-Policy}
\label{sec:method_main}
In this section, we propose our  \emph{Selective Importance Sampling} (SIS), a \textbf{simple yet effective method} building upon rejection sampling that directly converts off-policy tokens into on-policy ones at the token level, so that obviate the aforementioned problems brought by off-policy data. 
Beyond the SIS, we propose an efficient implementation of it to make it a plug-in with negligible efforts of existing algorithms. Finally, we prove the theoretical improvements of SIS in terms of provably tight sequence-versus-token approximation error bound.   
\begin{figure}[t!]
  \centering
\begin{center}
\vspace{-0.3in}
  \includegraphics[width=\linewidth]{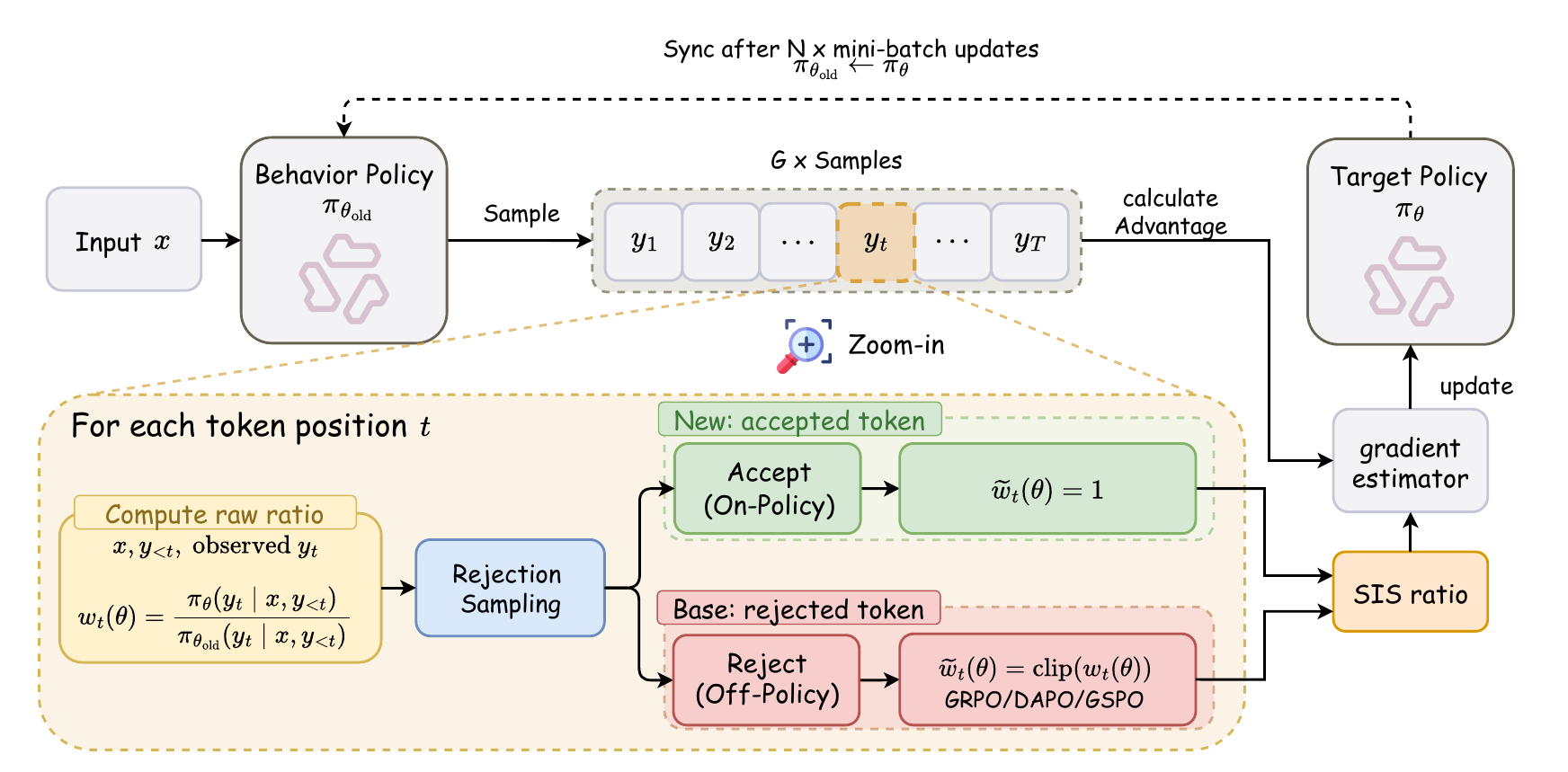}
\end{center}
  \vspace{-0.1in}
  \caption{Overview of SIS. At each token position, SIS applies a rejection sampling to determine whether the token can be treated as on-policy ($\widetilde{w}_t(\theta)=1$) or should retain its off-policy importance ratio. The modified ratios are then plugged into the gradient estimator of any existing algorithm.}
  \label{fig:sis}
\end{figure}

\vspace{-0.1in}
\subsection{SIS: Converting Off-Policy Tokens via Rejection Sampling}
\label{sec:rejection}

The high-level idea of SIS is straightforward (see Figure~\ref{fig:sis} for an illustration): at each token position, 
the behavior policy $\pi_{\theta_{\mathrm{old}}}$ can be viewed as a proposal distribution to sample from current policy $\pi_{\theta}$ under acceptance-rejection sampling. Then, by a acceptance-rejection test, the accepted tokens from behavior policy can be recognized from current policy, and become on-policy tokens with identical importance ratio; otherwise, the token retains its original importance ratio as in existing methods. 
In this way, SIS reduces the variance of importance ratios while preserving gradient signal.

Concretely, for each position $t$ with prefix $(x, y_{<t})$, 
we compute the constant 
\begin{equation}
\label{eq:envelope}
\small
M_t = \max_{v \in \mathcal{V}} \frac{\pi_\theta(v \mid x, y_{<t})}{\pi_{\theta_{\mathrm{old}}}(v \mid x, y_{<t})},
\end{equation}
where $\mathcal{V}$ denotes the vocabulary, $\pi_{\theta_{\mathrm{old}}}(\cdot \mid x, y_{<t})$, and $\pi_{\theta}(\cdot \mid x, y_{<t})$ are respectively proposal and target distributions. Given the observed token $y_t$ with ratio $w_t(\theta) = \frac{\pi_\theta(y_t \mid x, y_{<t})}{\pi_{\theta_{\mathrm{old}}}(y_t \mid x, y_{<t})}$, we draw an acceptance indicator
\begin{equation}\label{eq:accept}
\small
z_t \sim \mathrm{Bernoulli}\!\left(\frac{w_t(\theta)}{M_t}\right).
\end{equation}
By the standard theory of rejection sampling, acceptance ($z_{t} = 1$) certifies that $y_t$ follows exactly the target distribution $\pi_\theta(y_{t}\mid x, y<t)$. We state this process formally as below.

\begin{restatable}{proposition}{certificate}
\label{prop:certificate}
Let $y_t \sim \pi_{\theta_{\mathrm{old}}}(\cdot \mid x, y_{<t})$ and $z_t \sim \mathrm{Bernoulli}(w_t(\theta) \,/\, M_t)$ as defined in \eqref{eq:accept}. Then
\begin{equation}
\small
    \mathbb{P}(y_t = v \mid z_t = 1) = \pi_\theta(v \mid x, y_{<t}), \quad \forall\, v \in \mathcal{V}.
\end{equation}
\end{restatable}

Proposition~\ref{prop:certificate} states that, although $y_t$ is sampled from $\pi_{\theta_{\mathrm{old}}}$, its conditional distribution over $z_t=1$ becomes the distribution of $\pi_\theta$. In other words, accepted tokens are already on-policy samples and require no importance correction. 
Setting their weight to one therefore introduces \emph{zero bias}. 
Accordingly, the modified token-level ratio
\begin{equation}
\small
\widetilde{w}_t(\theta) =
\begin{cases}
1, & z_t = 1 \quad \text{(on-policy, weight removed)}, \\
w_t(\theta), & z_t = 0 \quad \text{(off-policy, weight retained)},
\end{cases}
\label{eq:modified_ratio}
\end{equation}
plugged in \eqref{eq:gradient} results in identical gradient estimation. 

We visualize the accepted rate of the rejection process in Figure~\ref{fig:accept}(A). As can be seen, the accept rate is \textbf{substantial} in both scenarios, demonstrating that a large fraction of tokens can be converted from off-policy to on-policy. 
This indicates that \emph{the modification introduced by SIS is no trivial}.
The lower accept rate in the agentic setting is expected, as interleaving reasoning with diverse tool calls introduces greater distributional shift between $\pi_{\theta_{\mathrm{old}}}$ and $\pi_\theta$.

\begin{figure}[t!]
  \centering
\begin{center}
  \includegraphics[width=\linewidth]{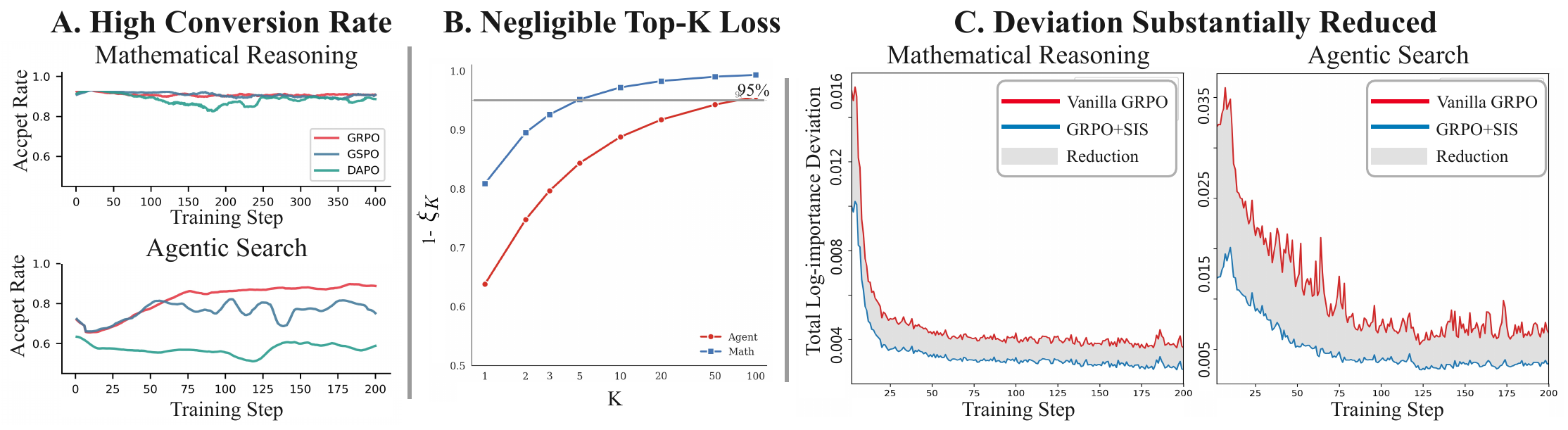}
\end{center}
  \vspace{-0.1in}
  \caption{Empirical validation of the three components of SIS on Qwen3-8B-Base. \textbf{(A)}~Accept rate of rejection sampling on math and agentic search; higher means more off-policy tokens are converted to on-policy (Appendix~\ref{app:accept_rate_extra} for 30B-MoE and 14B). \textbf{(B)}~Residual mass $\xi_K$ outside the top-$K$ set (upper) and cumulative coverage (lower) versus $K$; lower $\xi_K$ and coverage near $1.0$ indicate a tight top-$K$ envelope. \textbf{(C)}~Total log-importance deviation $D$ for vanilla GRPO vs.\ GRPO+SIS; a lower curve and larger shaded gap correspond to a tighter optimization bound.}
  \label{fig:accept}
\end{figure}
\begin{remark}
    In standard PPO algorithm, $w_{t}(\theta)$ is processed by hard clipping as in \eqref{eq:grpo}. This is originated from Trust-Region Policy Optimization (TRPO) \citep{schulman2015trust,qi2026rethinking}, which regularize the KL divergence between behavior policy and current policy. However, as pointed in \citep{qi2026rethinking}, the clipping heavily relies on the sampled token, which introduces large noise in determining trust-region. However, our rejection process distributionally compare the behavior policy and current policy, so that also mitigate the noise introduced by sampling process. 
\end{remark}

\subsection{Practical Top-$K$ Approximation}
\label{sec:approximation}
While the above rejection sampling procedure provides an unbiased mechanism for turning off-policy tokens on-policy, computing the maximal constant $M_t$ in \eqref{eq:envelope} requires maximizing over the entire vocabulary, which is computationally prohibitive. 
To mitigate this, we introduce a top-$K$ approximation to this constant.

Our method is built upon the fact that the probability mass of LLM is concentrated on a small subset of vocabulary, so that we can approximate $M_{t}$ with the maximization restricted to the top-$K$.
Concretely, we define
\begin{equation}
\small
    \widehat M_t = \max_{v \in \cV_K} \frac{\pi_\theta(v \mid x, y_{<t})}{\pi_{\theta_{\mathrm{old}}}(v \mid x, y_{<t})}, \quad \cV_K = \mathrm{TopK}(\pi_{\theta_{\rm old}}(\cdot\mid x, y_{<t}),\, K),
\end{equation}
and replace $M_{t}$ in \eqref{eq:accept} with it. Under this approximation, when computing acceptance indicator $z_{t}$ \eqref{eq:accept}, only tokens within $\cV_K$ are considered during rejection sampling, while tokens outside this set are still off-policy data. The approximated algorithm is summarized in Algorithm~\ref{alg:sis} in Appendix~\ref{app:algorithm}. In this regime, we have the following proposition. 
\begin{restatable}{proposition}{topkapprox}
\label{prop:topk_approx}
When approximating SIS with Algorithm~\ref{alg:sis}, the accepted tokens follows conditional distribution $\hat\pi_\theta(v \mid x, y_{<t}) = \pi_{\theta}(v\mid x, y_{<t}, v\in \cV_{K})$. Beside that, 
let $\xi_K = \sum_{v \notin \cV_K} \pi_\theta(v \mid x, y_{<t})$ denote the probability mass of the target policy outside the $\cV_{K}$, then we have 
\begin{equation}
\small
    D_{\rm TV}\!\left(\hat\pi_\theta(\cdot \mid x, y_{<t}) \parallel \pi_\theta(\cdot \mid x, y_{<t})\right) = \xi_K,
\end{equation}
where $D_{\rm TV}$ is the total variation distance. 
\end{restatable}
The proposition indicates that the ``on-policy'' data obtained by our efficiency-aware implementation follows a conditional distribution, which is close to the target when $\xi_{K}\to 0$. Since the top-$K$ tokens of $\pi_{\theta}$ and $\pi_{\theta_{\rm old}}$ will be overlapped with the increasing of $K$, and only a small number of high-probability tokens concentrate the vast majority of the probability mass, the $\xi_K$ remains small in practice, even for a relative small $K$, e.g., $K = 5$ (see Figure~\ref{fig:accept}(B)). 

\textbf{Complexity analysis.}
Beyond approximation quality, the top-$K$ formulation also keeps computational cost low. 
SIS reuses the old-policy logits from generation and the current-policy logits from training, requiring no extra model forward pass.
The only additional work per token is selecting the top-$K$ candidates under the old policy 
and computing their acceptance probabilities from cached logits, which adds only $\sim 1\%$ wall-clock overhead per training step (see Appendix~\ref{app:overhead}).

\subsection{Why SIS Works: A Theoretical Perspective}
\label{sec:theoretical_analysis}

The preceding subsections established \textbf{what} SIS does: selectively resetting off-policy tokens to on-policy via rejection sampling (\S\ref{sec:rejection}), and \textbf{how} to implement it efficiently via a top-$K$ approximation process (\S\ref{sec:approximation}). We now turn to \textbf{why} this mechanism is effective by revisiting the fundamental question raised in \eqref{eq:gradient} of Section~\ref{sec:pre}: \emph{how does the gap between the token-level and sequence-level gradient estimates grow as the behavior policy drifts, and can SIS provably reduce it?}

To answer this, we first characterize the approximation error $E = \|\bg_{\rm{seq}} - \bg_{\rm{tok}}\|$ between the ground-truth gradient $\bg_{\rm{seq}}$ \eqref{eq:ground truth gradient} with sequence-level correction ratio \eqref{eq:ratio} $w(\theta)$ and its surrogate $\bg_{\rm{tok}}$ \eqref{eq:gradient} with token-level correction ratio $w_{t}(\theta)$ \eqref{eq:ratio}. 

\begin{restatable}{theorem}{errorbound}
\label{thm:error_bound}
For a sequence of length $T$, we have 
\begin{equation}\label{eq:error bound}
\small
E \le |A(x, y)|\, (e^D - 1) \sum_{t=1}^T \|s_t\|,
\end{equation}
where $D = \sum_{t=1}^T |\!\log w_t(\theta)|$, $A(x, y)$ is the advantage and $s_t = \nabla_\theta \log \pi_\theta(y_t \mid x, y_{<t})$.
\end{restatable}
The proposition is proved in Appendix~\ref{app:proof_approx}. 
As can be seen, the total log-importance deviation $D = \sum_{t=1}^T |\!\log w_t(\theta)|$ serves as the key quantity governing the approximation error $E$, which is consistent to the fact that for the gradient on each token, the real correction $w(\theta) = \prod_{t=1}^{T}w_{t}(\theta)$ is replaced with $w_{t}(\theta)$. 
This exponential dependence explains why off-policy training destabilizes rapidly as the behavior policy drifts. The drift makes $w_{t}(\theta)\not\approx 1$ ($1 \leq t\leq T$) so that $w_{t}(\theta)\not\approx w(\theta) = \prod_{t=1}^{T}w_{t}(\theta)$. Next, we use a proposition to illustrate that estimating the policy gradient with our SIS strictly reduces $D$, thereby tightening the error bound \eqref{eq:error bound}.
\begin{restatable}{proposition}{errorreduction}
\label{prop:error_reduction}
Under the same conditions as Theorem~\ref{thm:error_bound}, replacing the correction ratios $w_t(\theta)$ in \eqref{eq:gradient} with the ratios $\widetilde{w}_t(\theta)$  \eqref{eq:modified_ratio} in SIS yields
\begin{equation}
\label{eq:sis_bound}
\small
E \le |A(x,y)|\,(e^{D_{\mathrm{SIS}}} - 1) \sum_{t=1}^T \|s_t\|
\le |A(x,y)|\,(e^{D} - 1) \sum_{t=1}^T \|s_t\|,
\end{equation}
where $D_{\mathrm{SIS}} = \sum_{t=1}^T \,|\!\log\widetilde w_t(\theta)|= \sum_{t=1}^T (1-z_t)\,|\!\log w_t(\theta)|\leq D$ for $z_{t}$ defined in \eqref{eq:accept}.
\end{restatable}
The proof to this proposition is in Appendix~\ref{app:proof_approx}. Intuitively, the correction ratio $w(\theta) = \prod_{t = 1}^{T}w_{t}(\theta)$ is contributed by each off-policy token $w_{t}(\theta)$. However, our SIS transfers most off-policy tokens (see Figure \ref{fig:accept}) into on-policy tokens, which strictly reduces the bias $D$ of Theorem \ref{thm:error_bound} into  $D_{\mathrm{SIS}}$ in Proposition \ref{prop:error_reduction}. 
Figure~\ref{fig:accept}(C) corroborates this empirically: SIS consistently maintains a lower deviation than vanilla GRPO on both mathematical reasoning and agentic search tasks, directly translating to a tighter approximation error bound and more stable policy optimization.

\subsection{Plugging SIS into Existing Algorithms}
\label{sec:instantiation}
Notably, our SIS is algorithm-agnostic: it only modifies the token-level correction ratio $w_t(\theta) \to \widetilde{w}_t(\theta)$ (cf.\ \eqref{eq:modified_ratio}), while leaving the rest of the components unchanged. Thus it can be combined with any RL-post training algorithms implemented on off-policy data e.g., GRPO \citep{shao2024deepseekmath}, DAPO \citep{yu2025dapo}, GSPO \citep{zheng2025group}. Below we illustrate the combination with with GRPO; instantiations with more algorithms are given in Appendix~\ref{app:instant}. For a prompt $x$ from query set $\cD$ and a group of $G$ responses $\{y_i\}_{i=1}^G \sim \pi_{\theta_{\mathrm{old}}}(\cdot \mid x)$, we write $\widetilde{w}_{i,t}(\theta)$ for the SIS-modified ratio of the $t$-th token in $y_i$.

\textbf{GRPO + SIS.}
Applying our proposed SIS to GRPO~\citep{shao2024deepseekmath} yields
\begin{equation}\label{eq:grpo}
    \small
\begin{aligned}
    \cJ_{\rm GRPO}^{\rm SIS}(\theta) &= \mE_{x\sim \cD, \{y_i\}_{i=1}^G\sim\pi_{\theta_{\rm old}}(\cdot\mid x)} \\ &\left[\frac{1}{G}\sum_{i=1}^G \frac{1}{|y_i|} \sum_{t=1}^{|y_i|} \min\!\left(\widetilde{w}_{i,t}(\theta) \widehat{A}_{i},\; \mathrm{clip}(\widetilde{w}_{i,t}(\theta), 1-\varepsilon, 1+\varepsilon) \widehat{A}_{i} \right) - \beta D_{\rm KL}(\pi_\theta \parallel \pi_{\rm ref})\right],
\end{aligned}
\end{equation}
where $\widehat{A}_{i}$ is the standard group-based normalized advantage. The only modification is substituting $w_{i,t}(\theta)\to\widetilde{w}_{i,t}(\theta)$, confirming the plug-in nature of SIS.

\section{Experiments}

We empirically verify SIS along three axes: (1) consistent plug-in gains across policy-gradient algorithms, backbones, and tasks; (2) improved training stability under challenging off-policy regimes; (3) robustness to its hyperparameter and interpretable token-level selectivity.

\vspace{-2mm}
\subsection{Experimental Setup}
\label{sec:exp_setup}

\textbf{Models and Datasets.}
We conduct experiments with three Qwen3~\citep{yang2025qwen3} backbones: the dense Qwen3-8B-Base, Qwen3-14B-Base, and the Mixture-of-Experts (MoE) Qwen3-30B-A3B-Base. We mainly explore our results on math \citep{shao2024deepseekmath} and reasoning with search agent \citep{jin2025search}, and train our models on DAPO-Math-17K corpus~\citep{yu2025dapo} and the merged training sets of NQ~\citep{kwiatkowski2019natural} and HotpotQA~\citep{yang2018hotpotqa}, respectively. More experiments on Llama3.2 \citep{grattafiori2024llama} backbones are in Appendix \ref{app:llama}.

\textbf{Evaluation.}
We evaluate on 10 benchmarks across three settings:
\textbf{Math}: MATH500~\citep{hendrycks2021measuring}, AMC23~\citep{ouyang2022training}, AIME24 and AIME25~\citep{maa2024aime};
\textbf{General QA (Agent)}: Natural Questions (NQ)~\citep{kwiatkowski2019natural}, TriviaQA~\citep{joshi2017triviaqa}, and PopQA~\citep{mallen2022not};
\textbf{Multi-Hop QA (Agent)}: HotpotQA~\citep{yang2018hotpotqa}, Musique~\citep{trivedi2022musique}, and Bamboogle (Bamb)~\citep{press2023measuring}. The QA datasets are designed for reasoning with search agent. 
We report Avg@1 accuracy for most tasks and Avg@32 for the smaller AMC23, AIME24, and AIME25.
The experiments on math are conducted on Qwen3-8B-Base and Qwen3-30B-A3B-Base, and the experiments on agent are under Qwen3-8B-Base and Qwen3-14B-Base.

\textbf{Baselines.}
We integrate SIS into three representative RL objectives: token-level GRPO~\citep{shao2024deepseekmath} and DAPO~\citep{yu2025dapo}, and sequence-level GSPO~\citep{zheng2025group}.
We also compare against recent methods that address off-policy error through alternative mechanisms: clipping-based CISPO~\citep{chen2025minimax}, divergence-based DPPO-TV~\citep{qi2026rethinking}, and entropy-based Clip-Cov~\citep{cui2025entropy}.
\emph{Each SIS run is paired with its corresponding baseline under identical hyperparameters and infrastructure; the only modification is replacing $w_{i,t}(\theta)$ with the SIS-modified weight $\widetilde{w}_{i,t}(\theta)$ in \eqref{eq:modified_ratio}.}
All runs use FSDP~\citep{zhao2023pytorch} for training and vLLM~\citep{kwon2023efficient} for rollouts.
Detailed hyperparameters and descriptions of all comparison methods are provided in Appendix~\ref{app:hparams}.

\subsection{Main Results}
\label{sec:exp_main}
We summarize our main results in Tables \ref{tab:main} and \ref{tab:trick}. As can be seen in Table~\ref{tab:main}, our SIS consistently improves all policy-gradient algorithms on dense (Qwen3-8B/14B-Base) and MoE (Qwen3-30B-A3B-Base) architectures across math and agentic reasoning tasks, yielding average gains of up to $+6.37$ on math and $+2.71$ on agent benchmarks. 

On other hands,  Table~\ref{tab:trick} unifies recent policy-gradient-based methods, which are actually GRPO under varied $w_{i,t}(\theta) = g(w)$. 
SIS instead \emph{decomposes} in two-folds: accepted tokens are corresponded to $w_{i,t}(\theta) = 1$, while rejected tokens are handled by any base algorithm $g(\cdot)$.
This composability gives SIS a guaranteed floor and a high ceiling: the trivial $g(w)=w$ (no trick) already beats most methods, while pairing $g(\cdot)$ with a good algorithm e.g.., DAPO attains the best accuracy.

\begin{table}[t!]
\centering
\small
\caption{Performance of SIS as a plug-in across math and agent benchmarks. Shaded rows denote the application of SIS on top of each baseline. The best accuracy per model is \textbf{bolded}. $\Delta$ denotes the absolute gain over the corresponding baseline.}
\label{tab:main}

\resizebox{\linewidth}{!}{%
\begin{tabular}{l|ccccc c|ccccccc c}
\toprule
\multirow{2}{*}{Method} & \multicolumn{6}{c|}{Mathematical Reasoning} & \multicolumn{8}{c}{Agentic Search} \\
\cmidrule(lr){2-7} \cmidrule(lr){8-15}
& MATH500 & AMC23 & AIME24 & AIME25 & Avg & $\Delta$
& NQ & TriviaQA & PopQA & HotpotQA & Musique & Bamb & Avg & $\Delta$ \\
\midrule
\multicolumn{7}{c|}{\it Qwen3-8B-Base} & \multicolumn{8}{c}{\it Qwen3-8B-Base} \\
\midrule
GRPO       & 80.4 & 74.14 & 23.44 & 18.23 & 49.05 & -- & 49.14 & 66.99 & 44.68 & 46.24 & 19.24 & 48.8 & 45.85 & -- \\
\rowcolor{gray!12} \quad w/ SIS & \textbf{85.6} & 79.53 & 25.83 & 19.38 & 52.59 & +3.54 & 49.78 & 67.47 & \textbf{49.29} & \textbf{47.90} & 19.74 & \textbf{52.8} & \textbf{47.83} & +1.98 \\
\addlinespace[3pt]
DAPO       & 83.2 & 81.88 & 21.56 & 18.33 & 51.24 & -- & 49.61 & 65.69 & 45.95 & 42.89 & 16.17 & 47.2 & 44.59 & -- \\
\rowcolor{gray!12} \quad w/ SIS & 84.2 & \textbf{86.79} & 26.98 & 22.60 & \textbf{55.14} & +3.90 & \textbf{50.11} & 67.66 & 47.83 & 47.83 & 21.18 & 51.2 & 47.64 & +3.05 \\
\addlinespace[3pt]
GSPO       & \textbf{85.6} & 78.83 & 26.46 & 18.85 & 52.44 & -- & 49.72 & 67.53 & 46.93 & 47.12 & 20.52 & 45.6 & 46.24 & -- \\
\rowcolor{gray!12} \quad w/ SIS & 85.2 & 77.19 & \textbf{28.33} & \textbf{23.44} & 53.54 & +1.10 & 49.09 & \textbf{68.18} & 46.97 & 46.97 & \textbf{22.26} & \textbf{52.8} & 47.71 & +1.47 \\
\midrule
\multicolumn{7}{c|}{\it Qwen3-30B-A3B-Base} & \multicolumn{8}{c}{\it Qwen3-14B-Base} \\
\midrule
GRPO       & 82.8 & 78.59 & 25.52 & 18.23 & 51.29 & -- & 50.64 & 68.21 & 47.42 & 47.70 & 22.13 & 52.1 & 48.03 & -- \\
\rowcolor{gray!12} \quad w/ SIS & \textbf{85.8} & 83.91 & 35.94 & 25.00 & 57.66 & +6.37 & 50.86 & \textbf{70.13} & 48.54 & 47.94 & \textbf{23.91} & \textbf{56.2} & 49.60 & +1.57 \\
\addlinespace[3pt]
DAPO       & 84.2 & 81.02 & 29.17 & 18.96 & 53.34 & -- & 50.22 & 68.51 & 48.61 & 48.01 & 22.09 & 51.2 & 48.11 & -- \\
\rowcolor{gray!12} \quad w/ SIS & 85.2 & \textbf{84.77} & 35.10 & \textbf{27.40} & \textbf{58.12} & +4.76 & 51.55 & 69.25 & 50.05 & \textbf{49.98} & \textbf{23.91} & 54.4 & \textbf{49.86} & +1.75 \\
\addlinespace[3pt]
GSPO       & 75.0 & 77.34 & 34.90 & 25.52 & 53.19 & -- & 51.58 & 69.36 & 51.36 & 47.01 & 20.98 & 43.2 & 47.25 & -- \\
\rowcolor{gray!12} \quad w/ SIS & 85.4 & 79.38 & \textbf{35.41} & 24.38 & 56.14 & +2.95 & \textbf{52.22} & 68.22 & \textbf{51.58} & 47.47 & 21.02 & 46.4 & 47.82 & +0.57 \\
\bottomrule
\end{tabular}%
}
\end{table}

\begin{table}[h]
\centering
\small
\setlength{\tabcolsep}{4pt}
\caption{Trick-free training on Qwen3-8B-Base math benchmarks. All variants build on GRPO; the $w_{i,t}(\theta)$ column gives each method's effective coefficient (definitions of $\mathrm{TV}_{i,t}$, $\mathrm{Cov}_{i,t}$ in Appendix~\ref{app:trick_defs}). SIS sets $\widetilde w=1$ for accepted tokens and applies any trick $g(w)$ to the rejected ones. Best \textbf{bolded}, second-best \underline{underlined}; $\Delta$ is the relative Avg gain over GRPO.}
\label{tab:trick}
\resizebox{\linewidth}{!}{%
\begin{tabular}{l|l l|ccccc|c}
\toprule
Method & \multicolumn{2}{c|}{$w_{i,t}(\theta)$} & MATH500 & AMC23 & AIME24 & AIME25 & Avg & $\Delta$ (\%) \\
\midrule
\multicolumn{9}{l}{\emph{No correction}}\\
w/o IS & \multicolumn{2}{c|}{$1$} & 79.6 & 67.95 & 23.04 & 18.46 & 47.26 & -3.65 \\
w/o clip & \multicolumn{2}{c|}{$w$} & 78.6 & 72.34 & 22.82 & 18.75 & 48.13 & -0.92 \\
\midrule
\multicolumn{9}{l}{\emph{Baseline}}\\
GRPO & \multicolumn{2}{c|}{$\mathrm{clip}(w,\,1-\varepsilon,\,1+\varepsilon)$} & 80.4 & 74.14 & 23.44 & 18.23 & 49.05 & -- \\
\midrule
\multicolumn{9}{l}{\emph{Stabilization tricks}}\\
\quad + Clip-Higher (DAPO) & \multicolumn{2}{c|}{$\mathrm{clip}(w,\,1-\varepsilon_{\rm lo},\,1+\varepsilon_{\rm hi})$} & 83.2 & 81.88 & 21.56 & 18.33 & 51.24 & $+4.46$ \\
\quad + Seq-level IS (GSPO) & \multicolumn{2}{c|}{$\mathrm{clip}(s_i,\,1\pm\varepsilon),\ s_i=(\textstyle\prod_t w_{i,t})^{1/|y_i|}$} & \textbf{85.6} & 78.83 & 26.46 & 18.85 & 52.44 & $+6.91$ \\
\quad + Clipped-IS (CISPO) & \multicolumn{2}{c|}{$\mathrm{sg}[\min(w,\,1+\varepsilon_{\rm hi})]$} & 84.0 & \underline{85.55} & 26.46 & 21.77 & \underline{54.44} & $+11.0$ \\
\quad + Divergence (DPPO-TV) & \multicolumn{2}{c|}{$w\cdot\bone[\mathrm{TV}_{i,t}\le\delta]$} & 81.4 & 73.83 & 23.85 & 18.85 & 49.48 & $+0.88$ \\
\quad + Clip-Cov (Entropy) & \multicolumn{2}{c|}{$w\cdot\bone[\mathrm{Cov}_{i,t}\le\omega]$} & 80.8 & 76.56 & 24.17 & 18.23 & 49.94 & $+1.81$ \\
\midrule
\emph{Ours} & \emph{accept} & \emph{reject $g(w)$} & & & & & & \\
\quad + SIS (vanilla) & $1$ & $w$ & 82.2 & 76.72 & \textbf{28.33} & 20.00 & 51.81 & $+5.63$ \\
\quad + SIS (GRPO) & $1$ & $\mathrm{clip}(w,\,1\pm\varepsilon)$ & \textbf{85.6} & 79.53 & 25.83 & 19.38 & 52.59 & $+7.22$ \\
\rowcolor{gray!12} \quad + SIS (DAPO) & $1$ & $\mathrm{clip}(w,\,1-\varepsilon_{\rm lo},\,1+\varepsilon_{\rm hi})$ & 84.2 & \textbf{86.79} & \underline{26.98} & \underline{22.60} & \textbf{55.14} & $+12.4$ \\
\bottomrule
\end{tabular}%
}
\end{table}

\subsection{SIS Improves Training Stability}
\label{sec:exp_stable}

Next, we explore whether SIS improves training stability under challenging off-policy regimes, since it transfers most off-policy tokens into on-policy. 
We design three stress tests: (1)~increasing policy staleness by reusing rollouts for up to 16 gradient updates, (2)~amplifying train-inference mismatch via MoE routing divergence, and (3)~removing clipping entirely.
All comparisons keep the paired protocol from \S\ref{sec:exp_setup}.

\begin{figure}[t!]
  \centering
  \includegraphics[width=1.0\linewidth]{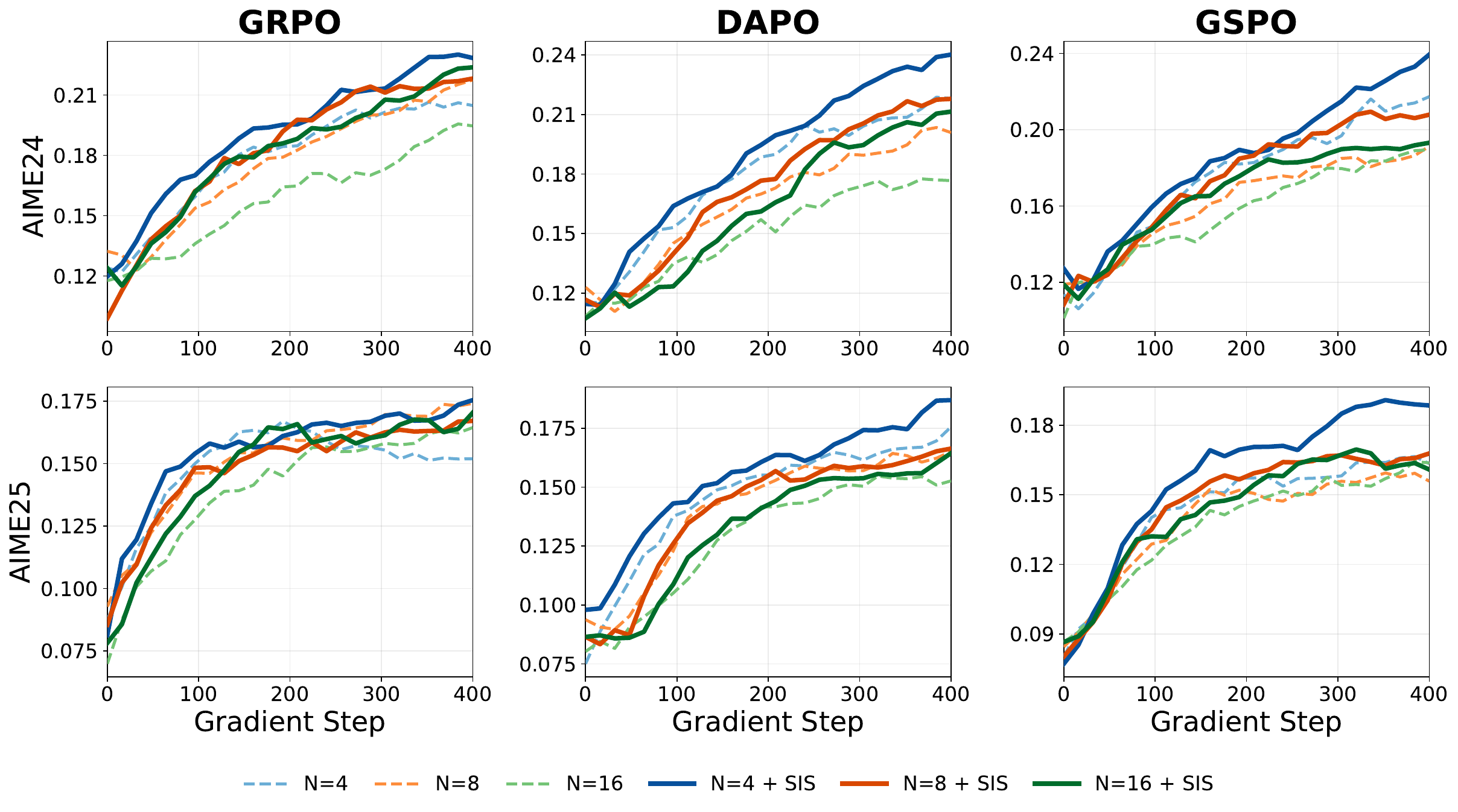}
  \vspace{-1.5ex}
  \caption{Training dynamics under increasing policy staleness ($N$=4, 8, 16) on Qwen3-8B-Base. Columns: GRPO, DAPO, GSPO. Rows: AIME24 and AIME25 accuracy. Dashed lines are baselines; solid lines add SIS. Increasing $N$ degrades the baselines, whereas the corresponding SIS variants stay consistently above their baselines at every staleness level.}
  \label{fig:N}
\end{figure}

\textbf{SIS resists degradation from stale rollouts.}
Reusing rollouts for multiple gradient updates is common practice to amortize generation cost, but increases the degree of off-policy-ness, and results in performance drop \citep{fu2026areal}. 
We vary $N$, the number of mini-batch updates per rollout, from 4 to 16 on Qwen3-8B-Base to explore  whether SIS remains effective as staleness grows.
Figure~\ref{fig:N} reveals two observations:
\emph{(i)~staleness is harmful in itself}: increasing $N$ degrades accuracy regardless of whether SIS is applied, confirming that reusing stale off-policy rollouts hurts RL, though increasing computational efficiency;
\emph{(ii)~SIS yields a stable gain at every staleness level}: each SIS variant stays consistently above its corresponding baseline across all $N$, because SIS maintains an appreciable accept rate even under large $N$ (Appendix~\ref{app:staleness}). 
Together, these results show that SIS reliably recovers performance under stale rollouts, persisting under longer training (detailed in Tab.~\ref{tab:staleness}).

\textbf{SIS handles train-inference mismatch in MoE models.}
In MoE architectures, the rollout engine and the trainer may route tokens through different experts, producing mismatched log-probabilities even at the same checkpoint~\citep{zheng2025stabilizing,IcePop2025}.
We test whether SIS can handle this system-level off-policy divergence on Qwen3-30B-A3B-Base, comparing against Rollout Routing Replay (R3)~\citep{ma2025stabilizing}, a method specifically designed for MoE routing mismatch.
As shown in Figure~\ref{fig:moe}, \emph{vanilla DAPO collapses entirely}: accuracy peaks and then degrades, entropy drops to near zero, and gradient norms spike by two orders of magnitude.
In contrast, both R3 and our SIS mitigate part of this instability, while \emph{combining SIS with R3 yields the strongest accuracy and the most stable optimization}, indicating that SIS and R3 address complementary sources of off-policy divergence.

\textbf{SIS enables clipping-free training.}
\label{sec:exp_trick}
SIS is itself a stabilization mechanism: with clipping fully disabled, GRPO+SIS still beats clip-based GRPO on math and agent tasks (Appendix~\ref{app:noclip}).

\subsection{Sensitivity and Token-Level Analysis}
\label{sec:exp_ablation}
We finally ablate the only hyperparameter of SIS and inspect its token-level acceptance behavior.

\textbf{Top-$K$ sensitivity.} 
SIS introduces a single hyperparameter $K$ that defines the rejection envelope. Sweeping $K\in\{10,50,100\}$ on Qwen3-8B-Base, Table~\ref{tab:ablation_topk} shows that \emph{SIS is robust to $K$}: average accuracy varies by at most $\sim$1 point, and we use $K{=}10$ throughout the paper.

\begin{figure}[t!]
  \centering
  \includegraphics[width=\linewidth]{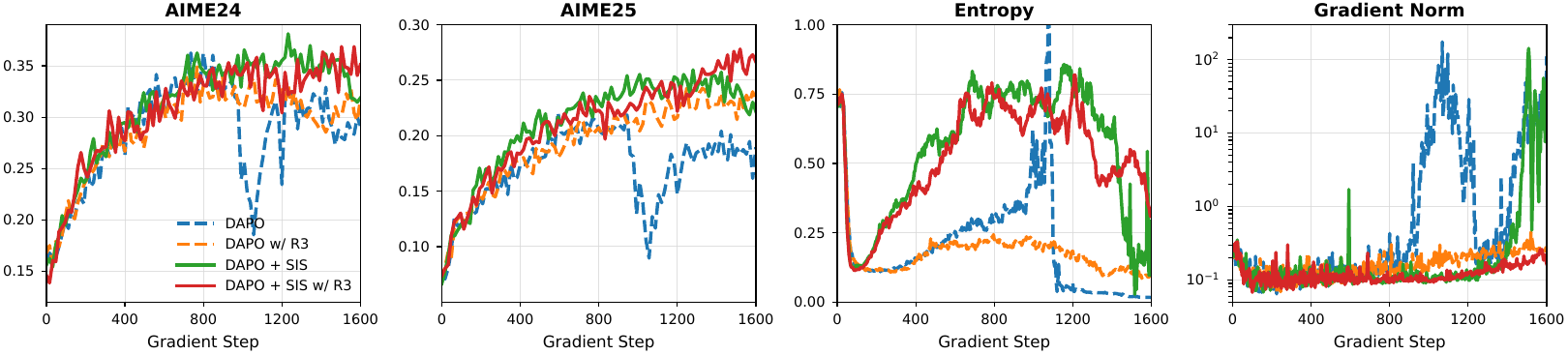}
  \vspace{-1.5ex}
  \caption{Training dynamics on Qwen3-30B-A3B-Base (MoE) with DAPO. Columns: AIME24/25 accuracy, entropy, gradient norm. Vanilla DAPO (blue dashed) suffers entropy collapse and gradient spikes; R3 (orange dashed) and SIS (green solid) each mitigate part of the instability; combining SIS with R3 (red solid) gives the strongest accuracy and the most stable optimization.}
  \label{fig:moe}
  \vspace{-0.5ex}
\end{figure}

\textbf{Token-level acceptance behavior.}
We further inspect what the rejection sampler accepts or rejects at the token level. Figure~\ref{fig:distinctive_tokens} visualizes the most distinctive tokens among accepted (green) and rejected (red) samples. Accepted tokens are dominated by mathematical reasoning terms (e.g., ``frac''), while rejected ones contain more formatting and web artifacts (e.g., ``\$image''), suggesting that SIS preferentially grants on-policy status to tokens aligned with the target reasoning distribution.

\begin{figure}[!htbp]
\centering
\small
\begin{minipage}[t]{0.49\linewidth}
\centering
\vspace{0pt}
\captionof{table}{Sensitivity to top-$K$ truncation on Qwen3-8B-Base.}
\label{tab:ablation_topk}
\vspace{-0.2ex}
\resizebox{\linewidth}{!}{%
\begin{tabular}{l|c|ccccc}
\toprule
Method      & $K$     & MATH500 & AMC23 & AIME24 & AIME25 & Avg   \\
\midrule
\multirow{3}{*}{GRPO + SIS}
            & 10      & 85.6    & 79.53 & 25.83  & 19.38  & 52.59 \\
            & 50      & 84.6    & 78.83 & 26.67  & 20.00  & 52.53 \\
            & 100     & 84.8    & 77.50 & 26.88  & 19.79  & 52.24 \\ \midrule
\multirow{3}{*}{DAPO + SIS}
            & 10      & 84.2    & 85.00 & 25.52  & 22.19  & 54.23 \\
            & 50      & 84.2    & 84.45 & 26.98  & 22.60  & 54.56 \\ 
            & 100     & 84.0    & 85.08 & 26.77  & 22.19  & 54.51 \\ \midrule
\multirow{3}{*}{GSPO + SIS}
            & 10      & 83.2    & 76.25 & 27.92  & 21.67  & 52.26 \\
            & 50      & 85.2    & 77.19 & 28.33  & 23.44  & 53.54 \\
            & 100     & 83.8    & 76.95 & 26.67  & 22.40  & 52.46 \\
\bottomrule
\end{tabular}
}
\end{minipage}%
\hfill
\begin{minipage}[t]{0.49\linewidth}
\centering
\vspace{0pt}
\includegraphics[width=0.49\linewidth]{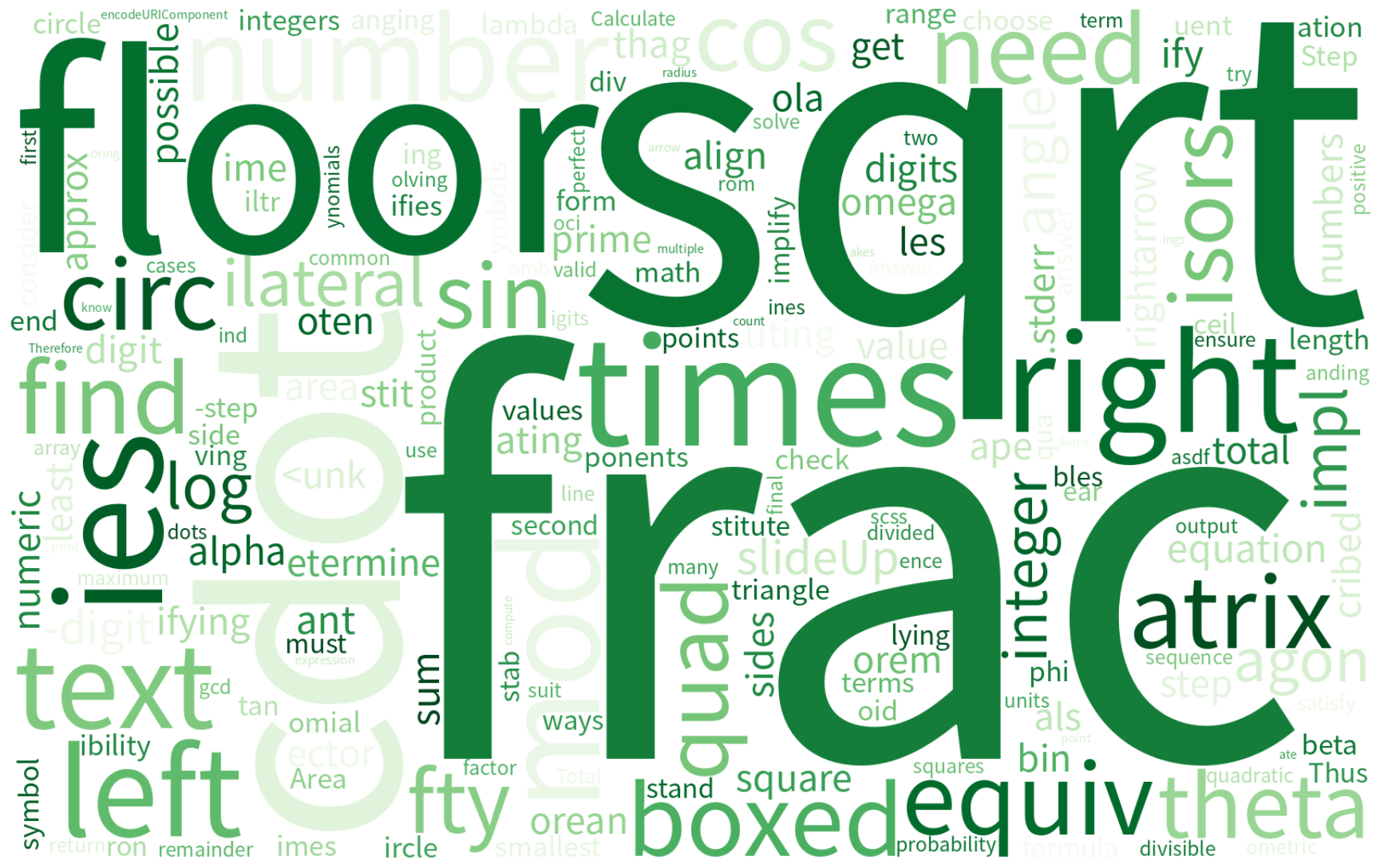}\hfill
\includegraphics[width=0.49\linewidth]{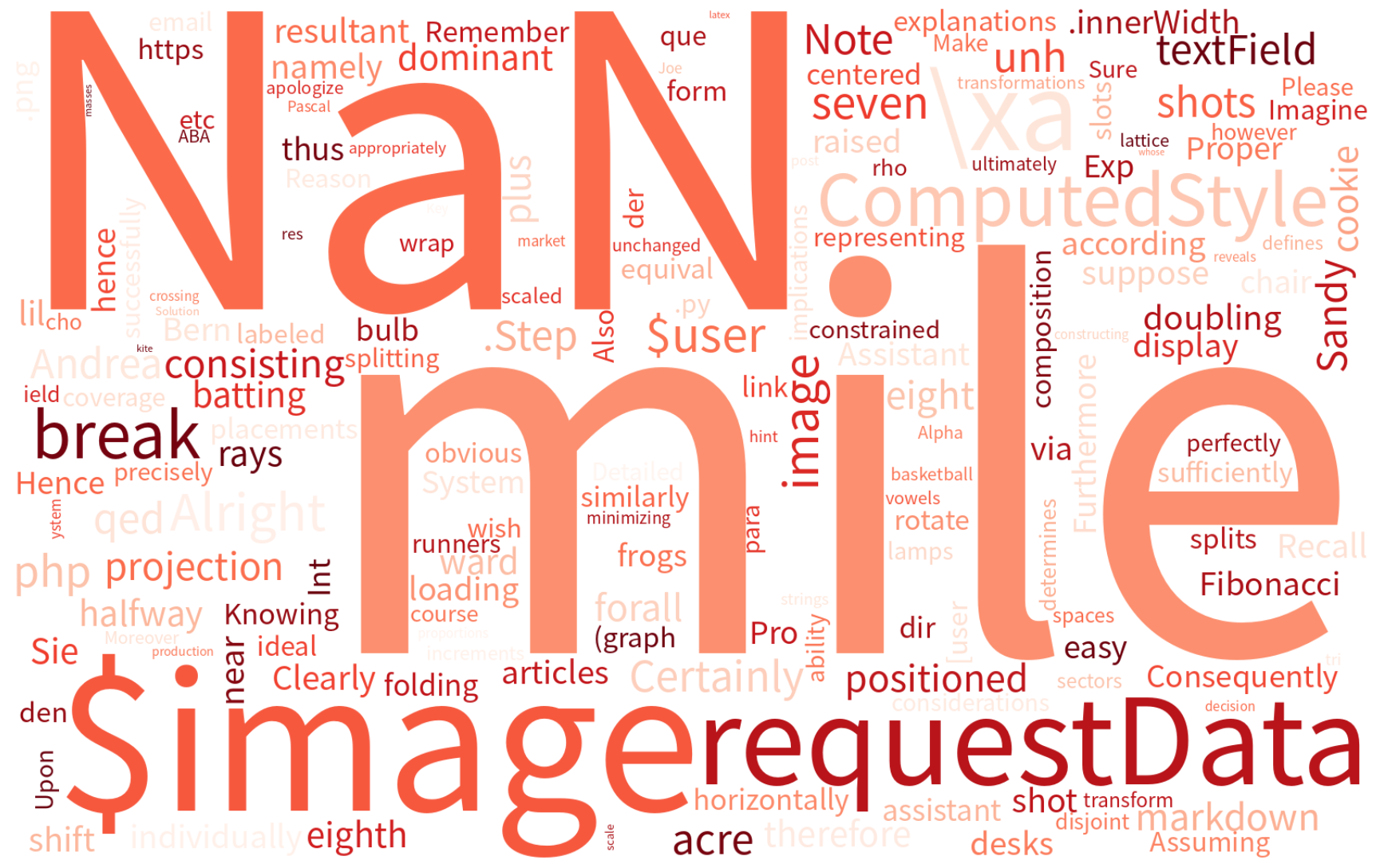}
\vspace{-1.0ex}
\caption{Token-level SIS decisions. Word clouds of distinctive tokens in accepted(left, green) vs.\ rejected(right, red) samples. SIS favors mathematical reasoning tokens and rejects formatting/web artifacts.}
\label{fig:distinctive_tokens}
\end{minipage}
\end{figure}

\section{Related Work}

\textbf{Importance Sampling Ratios for LLM RL.} Importance-sampling correction for off-policy training dates back to~\citet{schulman2015trust,sutton1998reinforcement}, but the per-token ratio introduces large variance and unstable training~\citep{schulman2017ppo,qi2026rethinking}. The dominant remedy is hard clipping~\citep{shao2024deepseekmath,yu2025dapo,cui2025entropy}, which bounds the ratio but discards gradient information from highly off-policy tokens; soft clipping~\citep{chen2025minimax,gao2025soft} retains gradients at the cost of sensitive hyperparameters, and other clipping variants~\citep{cagatan2026cfpo,leroux2025topr,lei2026minpro} follow the same recipe—shrinking large ratios to suppress variance, which inevitably \emph{attenuates the gradient signal}. SIS instead \emph{transfers} off-policy tokens to on-policy, fixing the off-policy mismatch at its root.

\textbf{Rejection Sampling and Sample Selection.}
Rejection sampling has appeared in the LLM pipeline only as a \emph{data curation} step rather than a sampling mechanism: RSO~\citep{liu2024statistical,tajwar2026maximum} filters preference pairs before reward-model training, best-of-$N$~\citep{gui2024bonbon,beirami2025bestofn,sun2024speculative,li2026ets} re-ranks candidates at inference, and other work removes OOD or unsafe samples from the training corpus~\citep{guo2025deepseek}. Unlike SIS, these response-level pre/post-processing steps do not modify the policy gradient.

\textbf{Stabilizing RL Post-Training.}
A parallel line of work stabilizes training at the \emph{systems} side, where staleness arises from asynchronous rollouts~\citep{fu2026areal,noukhovitch2025faster,gao2025beyond,team2025every} or from train--inference engine mismatch in MoE models~\citep{zheng2025stabilizing}. Common remedies all \emph{limit} the influence of off-policy data: masking high-deviation tokens~\citep{li2025trust}, reusing only sparse rollouts~\citep{luo2026sparse,wan2026bapo}, replaying inference-time routing via R3~\citep{ma2025stabilizing}, or filtering discrepant tokens via IcePop~\citep{IcePop2025}. SIS instead \emph{converts} off-policy tokens into on-policy, and is orthogonal to these tricks; in particular, combining SIS with R3 on MoE yields strictly stronger stability than either alone.

\vspace{-2mm}
\section{Conclusion}
\vspace{-2mm}
We present \emph{Selective Importance Sampling} (SIS) to tackle the off-policy bottleneck in post-training of LLMs, where stale rollouts and train--inference mismatch bias the policy gradient. 
Rather than \emph{suppressing} the correction ratio as in prior remedies, we reframe the problem as \emph{distribution conversion}. 
Concretely, SIS performs a token-level acceptance--rejection test with a top-$K$ approximation: accepted tokens are treated as fully on-policy, while rejected ones retain the standard IS correction. 
SIS only modifies the importance ratio, so it plugs into any policy-gradient objective with negligible overhead; it also provably tightens the off-policy approximation-error bound. 
Empirically, SIS delivers consistent gains across mainstream algorithms and dense/MoE backbones, improves stability under heavy rollout reuse and MoE routing mismatch. 

\bibliography{iclr2026_conference}
\bibliographystyle{iclr2026_conference}
\newpage
\appendix
\section{Proofs}
\label{app:proofs}

\subsection{Proof of On-policy Certificate}
\label{app:proof_certificate}

\certificate*
\begin{proof}

For any $v\in\cV$, applying Bayes' rule gives
\begin{equation}
\small
    \mathbb{P}(y_t = v \mid z_t = 1)
    = \frac{\mathbb{P}(z_t = 1 \mid y_t = v)\, \pi_{\theta_{\mathrm{old}}}(v \mid x, y_{<t})}{\mathbb{P}(z_t = 1)}.
\end{equation}
The acceptance probability conditioned on $y_t = v$ is
\begin{equation}
\small
    \mathbb{P}(z_t = 1 \mid y_t = v) = \frac{w_t(v)}{M_t} = \frac{\pi_\theta(v\mid x,y_{<t})}{M_t\, \pi_{\theta_{\mathrm{old}}}(v\mid x,y_{<t})}.
\end{equation}
The marginal acceptance probability is
\begin{equation}
\small
    \mathbb{P}(z_t = 1) = \sum_{v'\in\cV} \frac{\pi_\theta(v'\mid x,y_{<t})}{M_t\,\pi_{\theta_{\mathrm{old}}}(v'\mid x,y_{<t})} \cdot \pi_{\theta_{\mathrm{old}}}(v'\mid x,y_{<t}) = \frac{1}{M_t} \sum_{v'\in\cV} \pi_\theta(v'\mid x,y_{<t}) = \frac{1}{M_t}.
\end{equation}
Substituting back, we have
\begin{equation}
\small
    \mathbb{P}(y_t = v \mid z_t = 1) = \frac{\dfrac{\pi_\theta(v\mid x,y_{<t})}{M_t\,\pi_{\theta_{\mathrm{old}}}(v\mid x,y_{<t})} \cdot \pi_{\theta_{\mathrm{old}}}(v\mid x,y_{<t})}{1\,/\,M_t} = \pi_\theta(v\mid x,y_{<t}).
\end{equation}
\end{proof}

\subsection{Proof of Approximation Error}
\label{app:proof_approx}
\errorbound*
\begin{proof}
For a single response, the sequence-level gradient is
\begin{equation}
\small
\bg_{\mathrm{seq}} = A(x,y) \prod_{j=1}^T w_j(\theta) \sum_{t=1}^T s_t,
\end{equation}
while the token-level surrogate uses
\begin{equation}
\small
\bg_{\mathrm{tok}} = A(x,y) \sum_{t=1}^T w_t(\theta)\, s_t,
\end{equation}
where $s_t = \nabla_\theta \log \pi_\theta(y_t \mid x, y_{<t})$. Their discrepancy is
\begin{equation}
    \small
E = \left\| A(x,y) \sum_{t=1}^T \left(\prod_{j=1}^T w_j(\theta) - w_t(\theta)\right) s_t \right\| 
\le |A(x,y)| \sum_{t=1}^T \left|\prod_{j=1}^T w_j(\theta) - w_t(\theta)\right|\, \|s_t\|. \label{eq:triangle}
\end{equation}
Define $R = \prod_{j=1}^T w_j(\theta)$ and $\Delta_j = \log w_j(\theta)$, so that
$|R - w_t(\theta)| = w_t(\theta)\, \left|\prod_{j\neq t} w_j(\theta) - 1\right|$.
Applying the elementary inequality $|e^x - 1| \le e^{|x|} - 1$ to the product term gives
\begin{equation}
\small
\left|\prod_{j\neq t} w_j(\theta) - 1\right|
= \left| e^{\sum_{j\neq t} \Delta_j} - 1 \right|
\le e^{\left|\sum_{j\neq t} \Delta_j\right|} - 1
\le e^{\sum_{j\neq t} |\Delta_j|} - 1.
\end{equation}
Hence,
\begin{equation}
\small
|R - w_t(\theta)| \le e^{|\Delta_t|}\left(e^{\sum_{j\neq t} |\Delta_j|} - 1\right)
= e^{\sum_{j=1}^T |\Delta_j|} - e^{|\Delta_t|}
\le e^{D} - 1,
\end{equation}
where $D = \sum_{t=1}^T |\!\log w_t(\theta)|$. Plugging this bound into \eqref{eq:triangle} yields
\begin{equation}
\small
E \le |A(x,y)|\, (e^{D} - 1) \sum_{t=1}^T \|s_t\|.
\end{equation}
\end{proof}

\errorreduction*
\begin{proof}
For accepted tokens ($z_t=1$), $\widetilde{w}_t(\theta) = 1$ in the forward pass, so $\log \widetilde{w}_t(\theta) = 0$. Thus the total log-importance deviation becomes
\begin{equation}\small
D_{\mathrm{SIS}} = \sum_{t=1}^T |\!\log \widetilde{w}_t(\theta)| = \sum_{t=1}^T (1-z_t)\,|\!\log w_t(\theta)|.
\end{equation}
Defining $C = \sum_{t=1}^T z_t\,|\!\log w_t(\theta)| \ge 0$, we have $D_{\mathrm{SIS}} = D - C \le D$. Since $f(u) = e^u - 1$ is strictly increasing,
\begin{equation}\small
e^{D_{\mathrm{SIS}}} - 1 = e^{D-C} - 1 \le e^{D} - 1,
\end{equation}
with strict inequality if $C>0$. Applying Theorem~\ref{thm:error_bound} to the SIS gradient estimate gives the desired bound on $E_{\mathrm{SIS}}$.
\end{proof}

\subsection{Proof of top-$K$ Approximation}
\topkapprox*
\begin{proof}
Let $p(v) = \pi_\theta(v \mid x, y_{<t})$ denote the target distribution and $q(v) = \pi_{\theta_{\mathrm{old}}}(v \mid x, y_{<t})$ the proposal. Define $\cV_K = \mathrm{TopK}(q, K)$ and $\xi_K = \sum_{v \notin \cV_K} p(v)$.

Within $\cV_K$, the approximate envelope constant is $\widehat{M} = \max_{v \in \cV_K} \frac{p(v)}{q(v)}$. The acceptance probability for $v \in \cV_K$ is $\bbP(A=1 \mid v) = \frac{p(v)}{\widehat{M} \, q(v)}$, and for $v \notin \cV_K$, $\bbP(A=1 \mid v) = 0$.

The conditional distribution of accepted tokens is
\begin{equation}
\small
    \hat{\pi}_\theta(v) = \bbP(v \mid A=1) = \frac{\bbP(A=1 \mid v)\, q(v)}{\sum_{v' \in \cV_K} \bbP(A=1 \mid v')\, q(v')} = \frac{p(v) \,/\, \widehat{M}}{\sum_{v' \in \cV_K} p(v') \,/\, \widehat{M}} = \frac{p(v)}{\sum_{v' \in \cV_K} p(v')},
\end{equation}
for $v \in \cV_K$, and $\hat{\pi}_\theta(v) = 0$ for $v \notin \cV_K$, so that proves our first conclusion. Next, we compute the total variation distance
\begin{align}
\small
    D_{\rm TV}(\hat{\pi}_\theta \parallel p) &= \frac{1}{2} \sum_{v \in \cV} |\hat{\pi}_\theta(v) - p(v)| \nonumber \\
    &= \frac{1}{2} \sum_{v \in \cV_K} \left( \frac{p(v)}{1 - \xi_K} - p(v) \right) + \frac{1}{2} \sum_{v \notin \cV_K} p(v) \nonumber \\
    &= \frac{1}{2} \sum_{v \in \cV_K} p(v) \cdot \frac{\xi_K}{1 - \xi_K} + \frac{1}{2} \xi_K \nonumber \\
    &= \frac{1}{2} (1 - \xi_K) \cdot \frac{\xi_K}{1 - \xi_K} + \frac{1}{2} \xi_K = \xi_K.
\end{align}
\end{proof}

\section{Algorithm}
\label{app:algorithm}
We summarize the complete SIS procedure with the top-$K$ envelope approximation in Algorithm~\ref{alg:sis}. For each token position, the algorithm first identifies the top-$K$ most probable tokens under the behavior policy and restricts the rejection envelope to this set. Tokens falling inside $\mathcal{V}_K$ undergo an acceptance test: accepted tokens are certified as on-policy and receive unit importance weight, while rejected tokens and those outside $\mathcal{V}_K$ retain their original importance ratio. The resulting modified ratios are then directly substituted into any standard policy-gradient objective.

\begin{algorithm}[h]
\caption{SIS with Top-$K$ Envelope Approximation}
\label{alg:sis}
\begin{algorithmic}[1]
\REQUIRE Prompt $x$; responses $\{y_i\}_{i=1}^G \sim \pi_{\theta_{\mathrm{old}}}(\cdot \mid x)$; top-$K$ parameter $K$
\ENSURE Modified importance ratios $\{\widetilde{w}_{i,t}(\theta)\}$
\FOR{each response $y_i$, $i = 1, \ldots, G$}
    \FOR{each token position $t = 1, \ldots, |y_i|$}
        \STATE Compute token-level importance ratio $w_{i,t}(\theta)$
        \STATE Select top-$K$ set $\mathcal{V}_K$ under $\pi_{\theta_{\mathrm{old}}}(\cdot \mid x, y_{i,<t})$
        \STATE Compute approximation constant $\widehat{M}_t$ restricted to $\mathcal{V}_K$ \hfill $\triangleright$ Eq.~\eqref{eq:envelope}
        \IF{$y_{i,t} \in \mathcal{V}_K$}
            \STATE Accept token with probability $w_{i,t}(\theta)/\widehat{M}_t$; set $z_t \leftarrow 1$ if accepted \hfill $\triangleright$ Eq.~\eqref{eq:accept}
        \ELSE
            \STATE $z_t \leftarrow 0$ \hfill $\triangleright$ Tokens outside $\mathcal{V}_K$ remain off-policy
        \ENDIF
        \STATE Set $\widetilde{w}_{i,t}(\theta) \leftarrow 1$ if $z_t=1$, else retain $w_{i,t}(\theta)$ \hfill $\triangleright$ Eq.~\eqref{eq:modified_ratio}
    \ENDFOR
\ENDFOR
\STATE Plug $\{\widetilde{w}_{i,t}(\theta)\}$ into the policy-gradient objective \hfill $\triangleright$ e.g., GRPO, DAPO, GSPO
\end{algorithmic}
\end{algorithm}

\section{Experimental Setup Details}
\label{app:hparams}

\subsection{Math and Agent Experiments}

Table~\ref{tab:hparams_math} (math) and Table~\ref{tab:hparams_agent} (agent) summarize the hyperparameters. For each baseline, the vanilla run and the SIS run share the same training and evaluation hyperparameters; SIS only adds the top-$K$ approximation parameter used to compute the modified importance ratio. For agent training, we follow the settings of Search-R1 \citep{jin2025search}, using an E5 retriever \citep{wang2022text} and the 2018 Wikipedia dump \citep{karpukhin2020dense} as the corpus.

\begin{table}[h]
\small
\centering
\begin{minipage}[t]{0.49\linewidth}
\centering
\caption{Hyperparameters used in math experiments.}
\label{tab:hparams_math}
\resizebox{\linewidth}{!}{%
\begin{tabular}{l|c|c|c}
\toprule
Hyperparameter                      & GRPO                  & DAPO & GSPO    \\
\midrule
Learning rate                       & \multicolumn{3}{c}{$1\times 10^{-6}$} \\
Mini-batch size (mbs)               & \multicolumn{3}{c}{256}               \\
Responses per prompt                & \multicolumn{3}{c}{8}                 \\
Max prompt length                   & \multicolumn{3}{c}{1,024}             \\
Max response length                 & \multicolumn{3}{c}{3,072}             \\
Gradient steps                      & \multicolumn{3}{c}{1,600}             \\
Train temperature                   & \multicolumn{3}{c}{1.0}               \\
Eval temperature / top-$p$          & \multicolumn{3}{c}{1.0 / 0.7}         \\
Entropy loss coefficient            & \multicolumn{3}{c}{0}                 \\
Training engine                     & \multicolumn{3}{c}{FSDP}              \\
Inference engine                    & \multicolumn{3}{c}{vLLM}              \\
\midrule
KL loss coefficient                 & $1\times 10^{-3}$ & 0 & 0             \\
SIS top-$K$ (dense / MoE)           & 10 / 10 & 30 / 30 & 50 / 30                         \\
$\varepsilon_{\rm low}$             & 0.2 & 0.2 & $3\times 10^{-4}$                          \\
$\varepsilon_{\rm high}$            & 0.2 & 0.28 & $4\times 10^{-4}$              \\
\midrule
Global batch size (gbs)             & \multicolumn{3}{c}{$256 \times N$}    \\
GPUs                                & \multicolumn{3}{c}{32 (colocated)}    \\
\bottomrule
\end{tabular}%
}
\end{minipage}%
\hfill
\begin{minipage}[t]{0.49\linewidth}
\centering
\caption{Hyperparameters used in agent experiments.}
\label{tab:hparams_agent}
\resizebox{\linewidth}{!}{%
\begin{tabular}{l|c|c|c}
\toprule
Hyperparameter                      & GRPO                  & DAPO & GSPO    \\
\midrule
Learning rate                       & \multicolumn{3}{c}{$1\times 10^{-6}$} \\
Mini-batch size (mbs)               & \multicolumn{3}{c}{64}               \\
Responses per prompt                & \multicolumn{3}{c}{8}                 \\
Max prompt length                   & \multicolumn{3}{c}{4,096}             \\
Max response length                 & \multicolumn{3}{c}{4,096}             \\
Gradient steps                      & \multicolumn{3}{c}{800}             \\
Train temperature                   & \multicolumn{3}{c}{1.0}               \\
Eval temperature / top-$p$          & \multicolumn{3}{c}{1.0 / 1.0}         \\
Entropy loss coefficient            & \multicolumn{3}{c}{0}                 \\
Training engine                     & \multicolumn{3}{c}{FSDP}              \\
Inference engine                    & \multicolumn{3}{c}{vLLM}              \\
\midrule
KL loss coefficient                 & $1\times 10^{-3}$ & 0 & 0             \\
SIS top-$K$ (dense)                 & 50 & 50 & 50                        \\
$\varepsilon_{\rm low}$             & 0.2 & 0.2 & $3\times 10^{-4}$                          \\
$\varepsilon_{\rm high}$            & 0.2 & 0.28 & $4\times 10^{-4}$              \\
\midrule
Global batch size (gbs)             & \multicolumn{3}{c}{256}    \\
GPUs                                & \multicolumn{3}{c}{16 (colocated)}    \\
\bottomrule
\end{tabular}%
}
\end{minipage}
\end{table}

For all training-dynamics figures we plot only the first $400$ gradient steps: beyond that the policy has largely converged and $\pi_\theta\!\approx\!\pi_{\theta_{\rm old}}$, so the curves flatten and the off-policy gap of interest is no longer visible.

\subsection{Reward Design}
\label{app:reward}

All experiments use a purely outcome-based, rule-based reward; we apply no learned reward model, no format or length shaping, and no intermediate-step rewards. For a prompt $x$ and a generated response $y$, let $\hat a(y)$ be the answer extracted from $y$ and $a^\star$ (or $\mathcal{A}^\star$) denote the gold answer set.

\textbf{Math.}
We extract the final answer from the \texttt{\textbackslash boxed\{\}} field of $y$ and compare it against the reference with a symbolic equivalence checker $\mathrm{Eq}(\cdot,\cdot)$ that handles numeric and algebraic normalization. The reward is binary,
\begin{equation}
\small
r_{\rm math}(x,y) \;=\;
\begin{cases}
1, & \hat a(y)\ \text{is parseable from \texttt{\textbackslash boxed\{\}}}\ \text{and}\ \mathrm{Eq}\!\left(\hat a(y),\, a^\star\right)=\text{True}, \\
0, & \text{otherwise}.
\end{cases}
\end{equation}
Responses that omit a parseable \texttt{\textbackslash boxed\{\}} answer receive reward $0$.

\textbf{Agent.}
Following Search-R1~\citep{jin2025search}, the reward is the exact-match (EM) score of the final answer:
\begin{equation}
\small
r_{\rm agent}(x,y) \;=\; \mathrm{EM}(a_{\rm pred},\, a_{\rm gold}),
\end{equation}
where $a_{\rm pred}$ is the final answer extracted from the model response $y$ (enclosed in the \answer{$\cdot$} tags), and $a_{\rm gold}$ denotes the ground-truth answer. Search-engine invocations and retrieved passages are not rewarded directly; they affect the reward only through the final answer.

\subsection{Prompt Template}

For math reasoning, we adopt the same prompt format as Qwen-Math~\citep{yang2024qwen2}.
For agent reasoning, we follow Search-R1~\citep{jin2025search} and use a template that enforces a minimal yet sufficient structure without introducing content-specific biases.
The template organizes the model output into three iterative stages: (i) a reasoning phase, (ii) a search engine invocation phase, and (iii) a final answer.
We illustrate the templates in Table~\ref{tab:searchr1_template}.

\begin{table}[t]
\centering
\small
\caption{Prompt templates for math and agent tasks. The placeholder \texttt{question} is replaced by the actual query during both training and inference.}
\begin{tabular}{@{}l p{0.78\columnwidth}@{}}
\toprule
\textbf{Task} & \textbf{Prompt Template} \\
\midrule
Math & Please reason step by step, and put your final answer within \texttt{\textbackslash boxed\{\}}. \texttt{question}. \\
\midrule
Agent & Answer the given question. You must conduct reasoning inside \think{and} first every time you get new information.
After reasoning, if you find you lack some knowledge, you can call a search engine by \search{query}, and it will return the top searched results between \info{and}.
You can search as many times as you want.
If you find no further external knowledge needed, you can directly provide the answer inside \answer{and} without detailed illustrations.
For example, \answer{xxx}.
Question: \texttt{question}. \\
\bottomrule
\end{tabular}
\label{tab:searchr1_template}
\end{table}

\section{Additional Experimental Results}
\label{app:additional_exp}

\subsection{SIS Enables Stable Training Without Clipping}
\label{app:noclip}

To test whether SIS alone suffices for training stability, we disable clipping entirely in GRPO+SIS on Qwen3-8B-Base. Table~\ref{tab:noclip} shows that \emph{SIS without clipping still outperforms the clip-based GRPO} on both math and agent tasks, confirming that SIS is a stabilization mechanism in its own right, and one that does not sacrifice gradient information.

\begin{table}[!htbp]
\centering
\small
\caption{SIS does not require clipping on Qwen3-8B-Base. Best is \textbf{bolded}.}
\label{tab:noclip}
\resizebox{0.8\linewidth}{!}{%
\begin{tabular}{l|cccccc}
\toprule
\multicolumn{7}{c}{\textit{Math}} \\
\midrule
Method & MATH500 & AMC23 & AIME24 & AIME25 & \multicolumn{2}{c}{Avg} \\
\midrule
GRPO                   & 80.4 & 74.14 & 23.44 & 18.23 & \multicolumn{2}{c}{49.05} \\
\quad + SIS            & \bf{85.6} & \bf{79.53} & \bf{25.83} & \bf{19.38} & \multicolumn{2}{c}{\bf{52.59}} \\
\quad + SIS (w/o Clip) & 82.2 & 76.72 & 28.33 & 20.00 & \multicolumn{2}{c}{51.81} \\
\midrule
\multicolumn{7}{c}{\textit{Agent}} \\
\midrule
Method & NQ & TriviaQA & PopQA & HotpotQA & Musique & Bamb \\
\midrule
GRPO                   & 49.14 & 66.99 & 44.68 & 46.24 & 19.24 & 48.8 \\
\quad + SIS            & \bf{49.78} & 67.47 & 46.64 & \bf{47.90} & 19.74 & \bf{52.8} \\
\quad + SIS (w/o Clip) & 49.58 & \bf{67.90} & \bf{46.88} & 46.56 & \bf{20.94} & 52.0 \\
\bottomrule
\end{tabular}%
}
\end{table}

\subsection{Accept Rate Beyond Dense Math}
\label{app:accept_rate_extra}

To show that the on-policy certificate of SIS remains effective beyond the dense-math regime, we further inspect its token-level accept rate on the larger MoE backbone (Qwen3-30B-A3B-Base) and on the agentic search setting (Qwen3-14B-Base). As shown in Figure~\ref{fig:accept_extra}, SIS maintains an appreciable accept rate across all three policy-gradient algorithms in both settings. The accept rate is naturally lower on the agentic task, since interleaving reasoning with diverse tool calls induces a larger distributional shift between $\pi_{\theta_{\mathrm{old}}}$ and $\pi_\theta$; nevertheless it stays well above $0.6$, indicating that a substantial fraction of tokens remain effectively on-policy under $\pi_\theta$.

\begin{figure}[!htbp]
\centering
\includegraphics[width=0.9\linewidth]{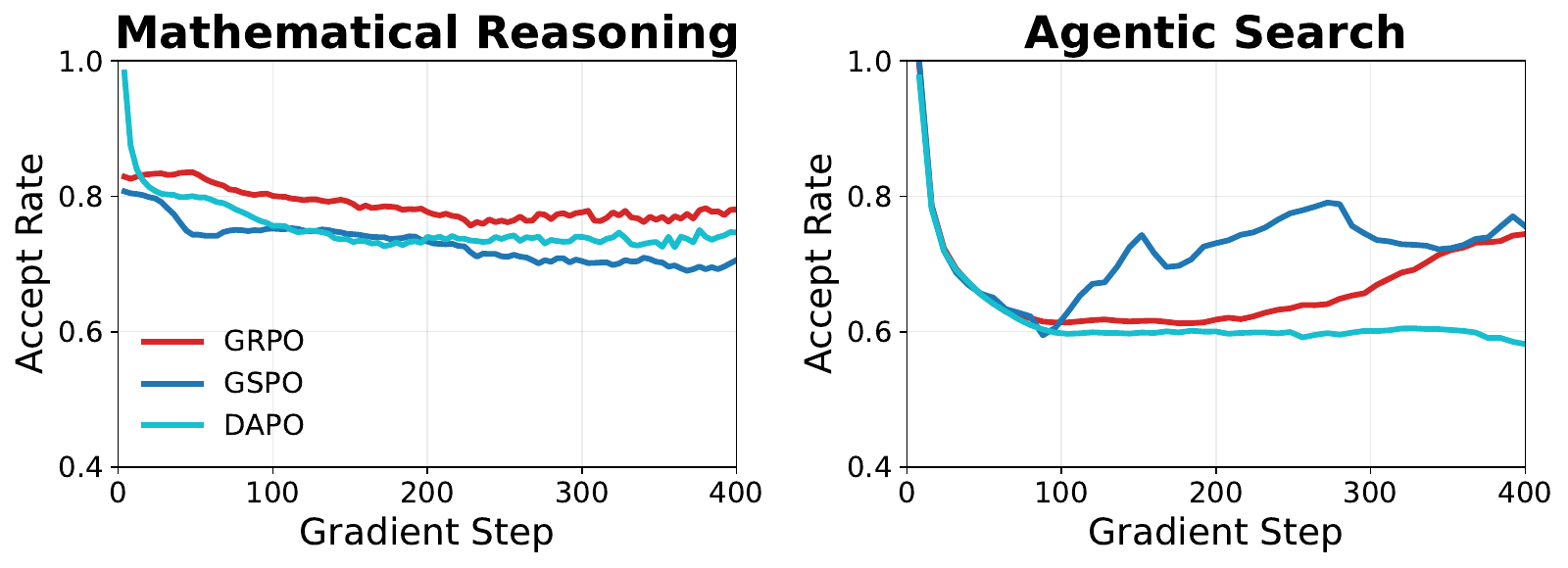}
\caption{Accept rate of SIS on Qwen3-30B-A3B-Base under math reasoning (left) and Qwen3-14B-Base under agentic search (right). SIS sustains an appreciable accept rate across all algorithms in both settings. The lower rate on the agentic task reflects the larger distributional shift introduced by interleaved tool calls between $\pi_{\theta_{\mathrm{old}}}$ and $\pi_\theta$.}
\label{fig:accept_extra}
\end{figure}

\subsection{Negligible Time Cost of SIS}
\label{app:overhead}

We profile the wall-clock cost introduced by SIS on top of vanilla GRPO. The setup follows our main math experiments: Qwen3-8B-Base, batch size $512$ with $n=8$ rollouts per prompt, response length $4{,}096$, trained on $8\times$NVIDIA A800-80GB GPUs. We instrument the five SIS-specific operations: the top-$K$ extraction during the \texttt{old\_log\_prob} stage, and the per-microbatch off-policy-to-on-policy correction during \texttt{update\_actor}. We sweep $K\in\{10, 50, 100\}$ to assess how the overhead scales.

Figure~\ref{fig:sis_overhead} decomposes the per-step wall-clock into rollout, $\pi_{\theta_{\mathrm{old}}}$ log-prob, actor update, and the additional SIS time. The total SIS overhead is $\sim 5.2$\,s per step, which corresponds to roughly $1\%$ of the $465$--$499$\,s step time across all three values of $K$. Furthermore, the overhead is essentially $K$-invariant: because the dominant cost is the off-policy-to-on-policy ratio reconstruction, which is fully parallelized on GPU and scales with $O(B\cdot L)$ rather than $K$, varying $K$ from $10$ to $100$ leaves the wall-clock virtually unchanged. The only price paid by a larger $K$ is GPU memory for the cached top-$K$ logits, not compute time.

\begin{figure}[!htbp]
\centering
\includegraphics[width=0.95\linewidth]{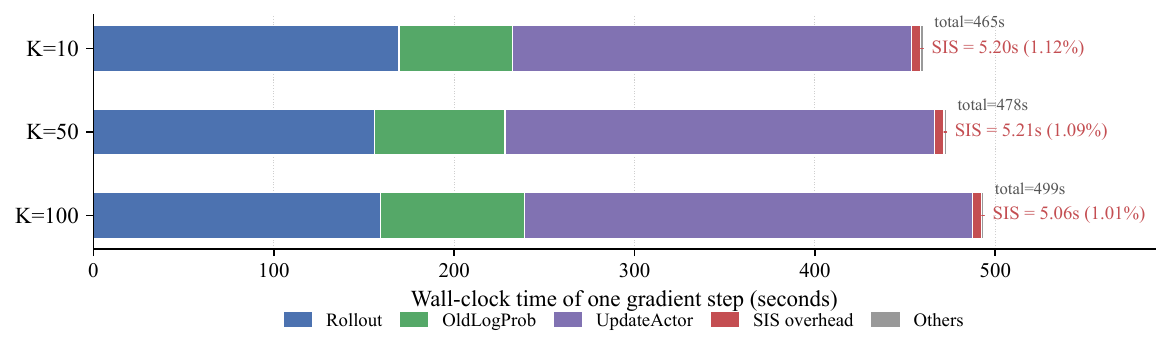}
\caption{Per-step wall-clock decomposition of GRPO+SIS on Qwen3-8B-Base across $K\in\{10,50,100\}$. The red segment is the SIS overhead: near-constant $\sim 5.2$\,s ($\sim 1\%$ of step time) and $K$-invariant.}
\label{fig:sis_overhead}
\end{figure}

\subsection{Training Dynamics}
\label{app:reward_dyn}

Beyond final accuracy and accept rate, we further visualize three training-side dynamics for all SIS runs in Table~\ref{tab:main}: token-level training reward, policy entropy, and gradient norm, summarized in Figure~\ref{fig:dyn_all}.

\textbf{Training reward.} The first column of Figure~\ref{fig:dyn_all} shows that the three policy-gradient algorithms paired with SIS all exhibit smooth, monotonically increasing reward curves on Qwen3-8B (math), Qwen3-30B-A3B-Base (math), Qwen3-8B (agentic search) and Qwen3-14B (agentic search), and converge to comparable levels within each setting. This indicates that SIS optimizes the training objective stably regardless of the underlying policy-gradient algorithm or the model scale.

\textbf{Policy entropy.} The middle column plots the policy entropy over the same horizon. Across all four settings, SIS keeps the entropy bounded away from zero, avoiding the entropy collapse commonly observed for aggressive off-policy training. On math tasks the entropy first drops as the policy sharpens on the training distribution, then stabilizes; GSPO+SIS in particular exhibits a mild early rebound before settling into a healthy range. On the agentic search tasks the entropy decreases more gradually due to the longer trajectories induced by tool calls, but again remains well above zero throughout training.

\textbf{Gradient norm.} The right column shows the gradient norm on a log scale. The gradient norm stays bounded within roughly $10^{-2}$--$10^{1}$ throughout training and does not diverge in any of the four settings. Occasional spikes appear on the agentic search runs, mostly on GSPO, reflecting the higher variance induced by variable-length tool-augmented rollouts; however these spikes remain within one order of magnitude of the median and do not lead to sustained instability, consistent with the smooth reward curves in the first column.

\begin{figure}[!htbp]
\centering
\includegraphics[width=0.9\linewidth]{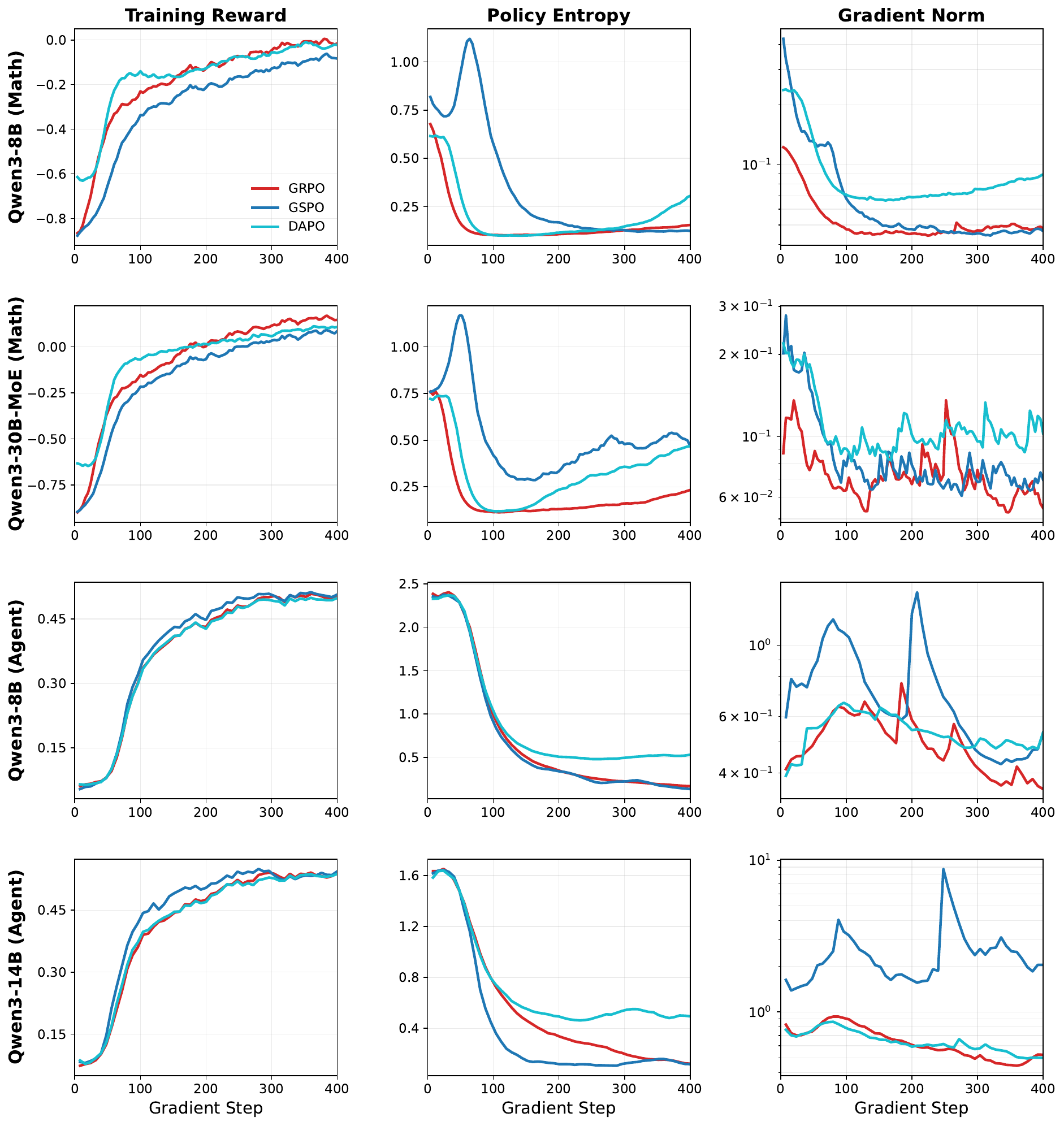}
\caption{Training dynamics of SIS over the first $400$ gradient steps for all entries in Table~\ref{tab:main}. Rows: Qwen3-8B/Qwen3-30B-A3B-Base on math and Qwen3-8B/Qwen3-14B on agentic search. Columns: training reward, policy entropy, and gradient norm (log scale). SIS produces stable reward growth and non-collapsed entropy across all backbones and algorithms; occasional gradient-norm spikes on agentic search runs are transient and do not lead to sustained instability.}
\label{fig:dyn_all}
\end{figure}

\subsection{Robustness Under Stale Rollouts}
\label{app:staleness}

We vary $N$, the number of mini-batch updates per rollout, to increase the drift from $\pi_{\theta_{\mathrm{old}}}$ to $\pi_\theta$ within each rollout cycle. Table~\ref{tab:staleness} reports final scores after $1600$ gradient steps under all three settings $N=4$, $8$, and $16$ grouped by algorithm, so the effect of staleness is directly comparable within each method. Two patterns emerge. First, staleness is harmful in itself: as $N$ grows from $4$ to $16$, average accuracy drops for every algorithm regardless of whether SIS is applied, confirming that reusing stale off-policy rollouts degrades RL. Second, SIS yields a stable gain at every staleness level: it improves over the corresponding baseline at all $N$ for GRPO, DAPO, and GSPO, and this advantage holds throughout the full $1600$-step training.

\textbf{Accept rate under staleness.} To further explain why SIS remains effective as $N$ grows, we plot the token-level accept rate during training in Figure~\ref{fig:accept_N}. Across all three algorithms and all staleness levels, the accept rate stays consistently in the $0.80$--$0.95$ range and never collapses. This confirms that even with $N=16$ a substantial fraction of tokens generated by $\pi_{\theta_{\mathrm{old}}}$ are still on-policy under $\pi_\theta$, so the on-policy certificate of SIS continues to provide a meaningful correction signal under high staleness.

\begin{figure}[!htbp]
\centering
\includegraphics[width=0.8\linewidth]{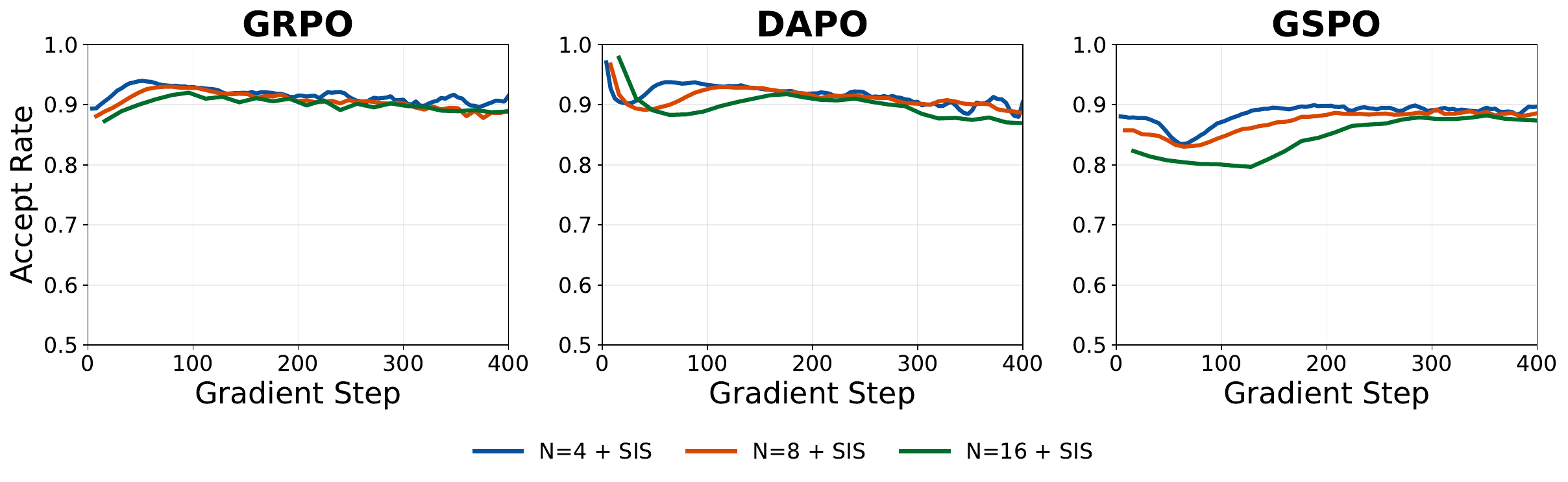}
\vspace{-1.5ex}
\caption{Token-level accept rate of SIS during the first $400$ gradient steps on Qwen3-8B-Base under staleness $N\in\{4,8,16\}$. Across all three algorithms, the accept rate stays in $0.80$--$0.95$ throughout training, indicating that most tokens remain effectively on-policy even at $N=16$.}
\label{fig:accept_N}
\end{figure}

\begin{table}[t]
\centering
\small
\caption{Performance of SIS under varying policy staleness on Qwen3-8B-Base after $1600$ gradient steps, where staleness $N$ is the number of mini-batch updates per rollout. Results are grouped by algorithm so that the effect of increasing $N$ is directly comparable within each method. Shaded rows denote the application of SIS on top of each baseline. Within each algorithm--staleness pair, the higher accuracy per metric is \textbf{bolded}.}
\label{tab:staleness}
\setlength{\tabcolsep}{5pt}
\begin{tabular}{cl|ccccc}
\toprule
$N$ & Method & MATH500 & AMC23 & AIME24 & AIME25 & Avg \\
\midrule
\multicolumn{7}{c}{\it GRPO} \\
\midrule
                    & Base         & 80.4 & 74.14 & 23.44 & 18.23 & 49.05 \\
\rowcolor{gray!12} \multirow{-2}{*}{4}  & \quad w/ SIS & \textbf{85.6} & \textbf{79.53} & \textbf{25.83} & \textbf{19.38} & \textbf{52.59} \\
\addlinespace[2pt]
                    & Base         & 81.6 & 71.33 & 22.91 & \textbf{20.73} & 49.14 \\
\rowcolor{gray!12} \multirow{-2}{*}{8}  & \quad w/ SIS & \textbf{83.8} & \textbf{78.59} & \textbf{24.08} & 20.10 & \textbf{51.64} \\
\addlinespace[2pt]
                    & Base         & 83.0 & 69.53 & 23.54 & 18.23 & 48.58 \\
\rowcolor{gray!12} \multirow{-2}{*}{16} & \quad w/ SIS & \textbf{84.6} & \textbf{77.03} & \textbf{24.69} & \textbf{21.35} & \textbf{51.92} \\
\midrule
\multicolumn{7}{c}{\it DAPO} \\
\midrule
                    & Base         & 83.2 & 81.88 & 21.56 & 18.33 & 51.24 \\
\rowcolor{gray!12} \multirow{-2}{*}{4}  & \quad w/ SIS & \textbf{84.2} & \textbf{86.79} & \textbf{26.98} & \textbf{22.60} & \textbf{55.14} \\
\addlinespace[2pt]
                    & Base         & 83.6 & 78.59 & \textbf{24.79} & 19.17 & 51.54 \\
\rowcolor{gray!12} \multirow{-2}{*}{8}  & \quad w/ SIS & \textbf{84.2} & \textbf{83.28} & 24.69 & \textbf{21.46} & \textbf{53.41} \\
\addlinespace[2pt]
                    & Base         & 83.4 & 73.83 & 20.42 & 17.60 & 48.81 \\
\rowcolor{gray!12} \multirow{-2}{*}{16} & \quad w/ SIS & \textbf{85.0} & \textbf{81.17} & \textbf{24.48} & \textbf{21.04} & \textbf{52.92} \\
\midrule
\multicolumn{7}{c}{\it GSPO} \\
\midrule
                    & Base         & \textbf{85.6} & \textbf{78.83} & 26.46 & 18.85 & 52.44 \\
\rowcolor{gray!12} \multirow{-2}{*}{4}  & \quad w/ SIS & 85.2 & 77.19 & \textbf{28.33} & \textbf{23.44} & \textbf{53.54} \\
\addlinespace[2pt]
                    & Base         & 83.4 & 73.75 & 25.31 & 18.02 & 50.12 \\
\rowcolor{gray!12} \multirow{-2}{*}{8}  & \quad w/ SIS & 83.4 & \textbf{78.20} & \textbf{25.52} & \textbf{21.77} & \textbf{52.22} \\
\addlinespace[2pt]
                    & Base         & 82.0 & 71.72 & 22.92 & \textbf{19.48} & 49.03 \\
\rowcolor{gray!12} \multirow{-2}{*}{16} & \quad w/ SIS & \textbf{83.8} & \textbf{73.83} & \textbf{23.85} & 19.27 & \textbf{50.19} \\
\bottomrule
\end{tabular}
\end{table}

\FloatBarrier

\subsection{Generalization Beyond Qwen}
\label{app:llama}
To verify that our findings generalize beyond a single model family, we further evaluate SIS on Llama-3.2-3B-Instruct~\citep{grattafiori2024llama}. We use the same DAPO-Math-17K training corpus and keep all hyperparameters identical to the Qwen3-8B-Base math setting (Table~\ref{tab:hparams_math}), ensuring a controlled comparison.

As presented in Table~\ref{tab:llama}, SIS again delivers consistent plug-in gains across all three algorithms, with the largest absolute lift on DAPO ($+4.47$ Avg). In relative terms the improvements are even more pronounced ($+11.0\%$ for GRPO and $+14.4\%$ for DAPO), notably larger than the corresponding relative gains on Qwen3-8B-Base. We attribute this to the smaller Llama model exhibiting a wider policy--rollout gap, making the off-policy correction provided by SIS even more impactful. These results confirm that SIS delivers robust gains independent of the underlying model architecture and family.

\begin{table}[t]
\centering
\small
\caption{Performance of SIS as a plug-in on Llama-3.2-3B-Instruct across math benchmarks. Shaded rows denote the application of SIS on top of each baseline. The best accuracy per column is \textbf{bolded}. $\Delta$ denotes the absolute gain over the corresponding baseline.}
\label{tab:llama}
\begin{tabular}{l|ccccc c}
\toprule
Method     & MATH500 & AMC23 & AIME24 & AIME25 & Avg & $\Delta$ \\
\midrule
GRPO       & 50.0 & 44.77 & 12.40 & 0.10 & 26.82 & -- \\
\rowcolor{gray!12} \quad w/ SIS & 51.2 & 53.36 & 13.75 & \textbf{1.77} & 29.76 & +2.94 \\
\addlinespace[3pt]
DAPO       & 57.2 & 52.74 & 13.96 & 0.21 & 31.03 & -- \\
\rowcolor{gray!12} \quad w/ SIS & \textbf{58.0} & \textbf{65.47} & \textbf{17.50} & 1.04 & \textbf{35.50} & +4.47 \\
\addlinespace[3pt]
GSPO       & 50.2 & 43.36 & 11.15 & 0.52 & 26.31 & -- \\
\rowcolor{gray!12} \quad w/ SIS & 50.6 & 43.36 & 13.94 & 0.63 & 27.13 & +0.82 \\
\bottomrule
\end{tabular}
\end{table}

\section{Stabilization Tricks Through the Importance-Sampling Lens}
\label{app:other_instantiations}
\label{app:trick_defs}

We organize Table~\ref{tab:trick} by writing every stabilization trick as a token-level rescaling of the importance-sampling (IS) coefficient
\begin{equation}
\small
w_{i,t}(\theta) \;=\; \frac{\pi_\theta(y_{i,t}\mid x,y_{i,<t})}{\pi_{\theta_{\rm old}}(y_{i,t}\mid x,y_{i,<t})}.
\end{equation}
Under this lens, each method amounts to a specific surrogate $g(w_{i,t}(\theta))$ used in place of $w_{i,t}(\theta)$ inside the policy-gradient objective. Clipping-based methods (GRPO, DAPO, GSPO, CISPO) reshape $w$ directly; DPPO-TV and Clip-Cov instead constrain the update through an auxiliary signal, which we recast as multiplying $w$ by an indicator that masks out-of-region tokens. We first restate the six representative methods in this unified form, and then use DAPO and GSPO as concrete cases to illustrate how SIS is instantiated on top of an existing trick.

\subsection{Representative stabilization methods}

\textbf{GRPO~\citep{shao2024deepseekmath}.}
The vanilla token-level PPO-style ratio with a symmetric clip bound $\varepsilon$,
\begin{equation}
\small
g_{\rm GRPO}(w) \;=\; \mathrm{clip}(w,\, 1-\varepsilon,\, 1+\varepsilon).
\end{equation}

\textbf{DAPO~\citep{yu2025dapo}.}
Decoupled clip bounds $\varepsilon_{\rm low}, \varepsilon_{\rm high}$ with the KL regularization removed,
\begin{equation}
    \small
\begin{aligned}
    \cJ_{\rm DAPO}(\theta) = ~&\mE_{(x,a)\sim \cD, \{y_i\}\sim\pi_{\theta_{\rm old}}} \\
    &\left[\tfrac{1}{\sum_{i}{|y_i|}}\!\sum_{i,t} \min\!\left(w_{i,t}(\theta)\widehat{A}_{i,t},\; \mathrm{clip}(w_{i,t}(\theta), 1-\varepsilon_{\rm low}, 1+\varepsilon_{\rm high}) \widehat{A}_{i,t} \right)\right],
\end{aligned}
\end{equation}
restricted to prompts with $0<|\{y_i\mid \texttt{is\_equivalent}(a,y_i)\}|<G$.

\textbf{GSPO~\citep{zheng2025group}.}
A sequence-level ratio obtained by averaging log-ratios,
\begin{equation}
    \small
    s_i(\theta) \;=\; \exp\!\left(\tfrac{1}{|y_i|}\sum_{t=1}^{|y_i|}\log \tfrac{\pi_\theta(y_{i,t}\mid x, y_{i,<t})}{\pi_{\theta_{\mathrm{old}}}(y_{i,t}\mid x, y_{i,<t})}\right),
\end{equation}
which then enters a sequence-level clipped objective with bounds $\varepsilon_{\rm low}, \varepsilon_{\rm high}$.

\textbf{CISPO~\citep{chen2025minimax}.}
A stop-gradient one-sided clip that detaches the high tail,
\begin{equation}
\small
g_{\rm CISPO}(w) \;=\; \mathrm{sg}\!\left[\min(w,\, 1+\varepsilon_{\rm hi})\right].
\end{equation}

\textbf{DPPO-TV~\citep{qi2026rethinking}.}
Instead of bounding $w$, it bounds the per-token total-variation distance between the current and behavior token distributions,
\begin{equation}
\small
\mathrm{TV}_{i,t}(\theta) \;=\; \tfrac12\big\|\pi_\theta(\cdot\mid x,y_{i,<t}) - \pi_{\theta_{\rm old}}(\cdot\mid x,y_{i,<t})\big\|_1 ,
\end{equation}
and discards the gradient of any token whose policy drift exceeds a threshold $\delta$, which is equivalent to the effective coefficient $w\cdot\bone[\mathrm{TV}_{i,t}\le\delta]$.

\textbf{Clip-Cov~\citep{cui2025entropy}.}
It attributes entropy collapse to a small set of tokens with large log-probability--advantage covariance,
\begin{equation}
\small
\mathrm{Cov}_{i,t}(\theta) \;=\; \big(\log\pi_\theta(y_{i,t}\mid x,y_{i,<t}) - \overline{\log\pi}\big)\big(\widehat A_{i,t} - \overline A\big) ,
\end{equation}
where $\overline{\log\pi}$ and $\overline A$ are batch means. It detaches the gradient of tokens whose covariance exceeds a threshold $\omega$, giving the effective coefficient $w\cdot\bone[\mathrm{Cov}_{i,t}\le\omega]$.

\subsection{Instantiating SIS: DAPO and GSPO as cases}
\label{app:instant}
Since SIS only replaces the token-level ratio $w_{i,t}(\theta)$ with its on-policy-equivalent counterpart $\widetilde{w}_{i,t}(\theta)$ (cf.\ \eqref{eq:modified_ratio}), instantiating SIS on top of any of the methods above amounts to a one-line substitution $w_{i,t}\!\to\!\widetilde w_{i,t}$ in the corresponding objective. We illustrate this with DAPO and GSPO, the two cases used in our experiments; analogous substitutions apply to the remaining methods and are left to future empirical study.

\paragraph{Case 1: DAPO + SIS.}
Substituting $w_{i,t}\!\to\!\widetilde w_{i,t}$ into the DAPO objective yields
\begin{equation}
    \small
\begin{aligned}
    \cJ_{\rm DAPO}^{\rm SIS}(\theta) = ~&\mE_{
    (x,a)\sim \cD, \{y_i\}_{i=1}^G\sim\pi_{\theta_{\rm old}}(\cdot\mid x)} \\&\left[\frac{1}{\sum_{i=1}^G{|y_i|}}\sum_{i=1}^G  \sum_{t=1}^{|y_i|} \min\!\left(\widetilde{w}_{i,t}(\theta)\widehat{A}_{i,t},\; \mathrm{clip}(\widetilde{w}_{i,t}(\theta), 1-\varepsilon_{\rm low}, 1+\varepsilon_{\rm high}) \widehat{A}_{i,t} \right)\right], \\
    \text{s.t.} ~&0 < \left|\{y_i\mid \texttt{is\_equivalent}(a,y_i) \}\right|<G,
\end{aligned}
\end{equation}
where $a$ is the ground-truth answer to the prompt $x$.

\paragraph{Case 2: GSPO + SIS.}
For sequence-level GSPO, accepted tokens contribute a zero log-ratio, giving the SIS-modified sequence ratio
\begin{equation}
\label{eq:gspo_is}
    \small
    \widetilde{s_i}(\theta) = \exp\!\left(\frac{1}{|y_i|}\left(\sum_{(i,t)\in \cA} \log \frac{\pi_\theta(y_{i,t}\mid x, y_{i,<t})}{\mathrm{sg}[\pi_\theta(y_{i,t}\mid x, y_{i,<t})]} + \sum_{(i,t)\notin \cA} \log \frac{\pi_\theta(y_{i,t}\mid x, y_{i,<t})}{\pi_{\theta_{\mathrm{old}}}(y_{i,t}\mid x, y_{i,<t})}\right)\right),
\end{equation}
which then enters the GSPO clipped objective
\begin{equation}
    \small
    \cJ_{\rm GSPO}^{\rm SIS}(\theta) = \mE_{x\sim \cD, \{y_i\}_{i=1}^G\sim\pi_{\theta_{\rm old}}(\cdot\mid x)}\left[\frac{1}{G}\sum_{i=1}^G \min\!\left(\widetilde{s_i}(\theta)\widehat{A}_i,\; \mathrm{clip}(\widetilde{s_i}(\theta), 1-\varepsilon_{\rm low}, 1+\varepsilon_{\rm high}) \widehat{A}_i \right)\right].
\end{equation}

\end{document}